\documentclass[journal]{IEEEtran}
\usepackage{amsmath,amsfonts,amssymb,bm}
\usepackage{array}
\usepackage[caption=false,font=normalsize,labelfont=sf,textfont=sf]{subfig}
\usepackage{textcomp}
\usepackage{stfloats}
\usepackage{placeins}
\usepackage{url}
\usepackage{graphicx}
\usepackage{lineno,hyperref}
\hypersetup{hidelinks,hypertexnames=false}
\usepackage{color}
\usepackage{threeparttable}
\usepackage{cite}
\usepackage{booktabs}
\usepackage{pifont}
\usepackage{multirow}
\usepackage{tabularx}
\usepackage{algorithm}
\usepackage{algorithmic}
\usepackage{makecell}
\newtheorem{proposition}{Proposition}
\newtheorem{remark}{Remark}
\newenvironment{proof}{\noindent\textit{Proof.}\ }{\hfill$\square$\par}

\newcommand{\method}{GeoRay}
\DeclareRobustCommand{\bench}{SatGauge}
\newcommand{\med}{\operatorname{med}}
\newcommand{\sg}{\operatorname{sg}}
\DeclareMathOperator*{\argmin}{arg\,min}
\providecommand{\best}[1]{\textbf{#1}}
\providecommand{\second}[1]{\underline{#1}}

\newcommand{\GaugeResidPx}{1.3{\times}10^{-4}}
\newcommand{\GaugeDispPx}{75.1}
\newcommand{\GaugeResidM}{7{\times}10^{-5}}
\newcommand{\GaugeResidAff}{0.051}
\newcommand{\GaugeResidLthree}{0.015}

\newcommand{\NullDraws}{500}
\newcommand{\NullMean}{0.008}
\newcommand{\NullSd}{0.342}
\newcommand{\CovOne}{87.0}
\newcommand{\CovTwo}{97.0}
\newcommand{\VarScale}{0.78}
\newcommand{\CovOneU}{86.4}
\newcommand{\CovTwoU}{97.0}
\newcommand{\VarScaleU}{0.79}
\newcommand{\CovOneAbs}{66.7}
\newcommand{\CovTwoAbs}{85.1}

\newcommand{\GainRhoWithin}{0.062}
\newcommand{\AnchorAbs}{3.63}
\newcommand{\AnchorPag}{68.9}
\newcommand{\AnchorCov}{87.2}
\newcommand{\AnchorRf}{86.1}
\newcommand{\BaselineKtenName}{MVSplat}
\newcommand{\BaselineKten}{6.12}
\newcommand{\BaselineShift}{7.21$ and $39.96}
\newcommand{\GeoCI}{[2.36,\,3.74]}
\newcommand{\PairBase}{SkySplat}

\newcommand{\CIabs}{[6.81,\,10.74]}

\begin{document}
	\title{GeoRay: Gauge-Aware Feed-Forward Satellite 3D Reconstruction in the Geodetic Frame}
	
	\author{
	Zhe~Dong,
	Wanqing~Wu,
	Yuzhe~Sun,
	Haochen~Jiang,
	Yuchen~Ma,
	Lecheng~Ren,
	Tianzhu~Liu,~\IEEEmembership{Member,~IEEE},
	and~Yanfeng~Gu,~\IEEEmembership{Senior Member,~IEEE}
	\thanks{This work was supported by the National Natural Science Foundation of China under Grant 624B2051. \emph{(Corresponding author: Yanfeng Gu.)}}
	\thanks{Z.~Dong, Y.~Sun, H.~Jiang, Y.~Ma, T.~Liu, and Y.~Gu are with the School of Electronics and Information Engineering, Harbin Institute of Technology, Harbin 150001, China (e-mail: guyf@hit.edu.cn).}
	\thanks{W.~Wu is with the 54th Research Institute, China Electronics Technology Group Corporation, Shijiazhuang 050081, China, and also with the State Key Laboratory of Comprehensive PNT Network and Equipment Technology, Shijiazhuang 050081, China.}
	\thanks{L.~Ren is with the School of Electrical and Electronic Engineering, University of Manchester, Manchester M13~9PL, U.K.}
}
	
	\maketitle
	
	\begin{abstract}
Feed-forward 3D foundation models reconstruct perspective scenes in one pass.
Satellite photogrammetry needs a different product, one that domain adaptation
alone does not deliver: dense surface height in an absolute geodetic frame
under non-central rational polynomial cameras (RPCs). Perspective-pretrained
features are not reliably observable along RPC height rays, absolute elevation
carries a low-order height--datum gauge exchangeable with sensor bias to first
order, and monocular and multi-view cues fail in different regions. \method{}
treats all three. Lightweight ray-consistent adapters make a frozen backbone
matchable along native RPC rays. An explicit datum mechanism separates relief
from absolute level and is equivariant to the vertical origin by construction,
so one trained model serves zero-, one-, and sparse-control inference.
Calibrated inverse-variance fusion combines the two relief streams. \bench{}, our absolute-frame
benchmark of eighteen systems across in-domain, cross-dataset, and cross-city
tiers, scores absolute placement without registration or test-reference leakage.
On 26 held-out US3D tiles, \method{} attains $2.99$\,m absolute MAE at
$91.9\%$ coverage, improves completeness-aware accuracy by $46.4$ points
over the strongest compliant feed-forward baseline, remains the most
accurate such system under both transfer shifts, and runs in $24$\,s
model-forward time per tile. Code and models will be released at
\url{https://github.com/HIT-SIRS/GeoRay}.
\end{abstract}
	
	
	\begin{figure*}[!t]
    \centering
    \includegraphics[width=\textwidth]{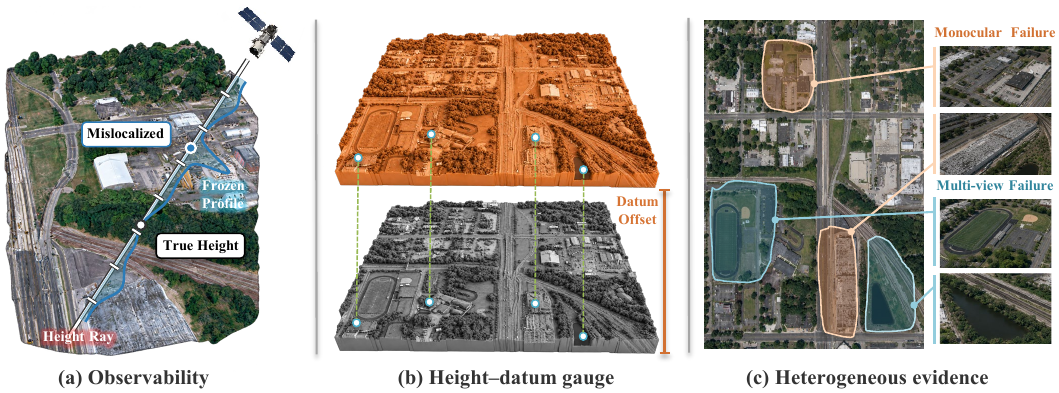}
    \caption{\textbf{Three structural obstacles to absolute satellite 3D.}
    (a) RPC ray-field observability, (b) the height--datum gauge, and
    (c) heterogeneous monocular and multi-view evidence motivate the three
    components of \method{}.}
    \label{fig:teaser}
\end{figure*}

\section{Introduction}
\label{sec:intro}

\IEEEPARstart{F}{eed-forward} 3D models now recover scene geometry from one or
many images in a single pass~\cite{wang2024dust3r,wang2025vggt,keetha2026mapanything},
while monocular foundation models provide strong metric or relative geometry
priors~\cite{yang2024depthanything2,wang2025moge2}. Their success, however,
has been established primarily under central-perspective image formation.
Satellite photogrammetry differs in ways that matter for deployment:
pushbroom sensors have no single projection center
~\cite{gupta1997pushbroom}, geometry is delivered as an RPC mapping
~\cite{tao2001rfm}, and the required product is a digital surface model whose
height is meaningful in an absolute geodetic frame. Existing high-accuracy
solutions therefore remain dominated by per-scene optimization
~\cite{defranchis2014s2p,mari2022satnerf,aira2025eogs} or
satellite-specific supervised models~\cite{gao2021satmvs,huang2026skysplat,yang2026sat3r}.
This paper asks whether feed-forward foundation geometry can be transferred to
native RPC imaging without sacrificing absolute elevation, completeness, or
calibrated uncertainty.

Three obstacles stand in the way. First,
\emph{ray-field observability}: perspective-pretrained features are not
trained to be compared at object-space points reached by a non-central RPC
height ray. Second, \emph{height--datum gauge freedom}: under the standard
low-order RPC bias model, a dominant component of vertical placement can be
exchanged with image-space sensor bias to first order
~\cite{grodecki2003block,fraser2003bias,triggs2000bundle}; accurate local
relief therefore does not imply correct absolute elevation. Third,
\emph{heterogeneous evidence}: monocular priors and multi-view matching fail
in different regions, so a fixed preference for either source is suboptimal.
Between them they determine what is observable, what stays ambiguous, and when
further evidence can improve a prediction.

\method{} handles all three within one feed-forward model. A frozen
VGGT backbone~\cite{wang2025vggt} is adapted with lightweight ray-consistent
low-rank modules so that cross-view evidence localizes correctly along RPC
height hypotheses. The resulting multi-view branch predicts relief and its
uncertainty, while a frozen MoGe-2 prior~\cite{wang2025moge2} supplies a
complementary monocular relief estimate. Absolute level is handled separately
by a datum mechanism whose response to a change of vertical origin is fixed
analytically, with only a small invariant residual learned from data. The two
relief streams are then combined by calibrated inverse-variance weighting.
One trained model then covers zero control, a single height control, and
sparse affine control.

The evaluation follows from the same analysis. Alignment hides the very
vertical placement error under study, test-reference height ranges leak
information into learned predictors, and error averages reward partial
surfaces. We introduce \bench{}, a protocol over public satellite datasets
that scores in the absolute geodetic frame, declares control budgets, excludes
reference-derived test inputs, and reports accuracy jointly with completeness
and relief fidelity across three transfer tiers.

On 26 held-out US3D tiles, \method{} reaches $2.99$\,m absolute MAE at $91.9\%$
coverage and $72.6\%$ completeness-aware accuracy at $2.5$\,m, a $46.4$-point
margin over the strongest compliant feed-forward baseline. The same model
remains the most accurate compliant feed-forward system under
cross-dataset and cross-city
transfer. Mechanism studies show that ray-consistent adaptation reduces
height-hypothesis mislocalization, that the measured camera gauge accounts for
a large share of the absolute error of pinhole-native feed-forward baselines,
and that predicted monocular uncertainty orders where multi-view evidence helps
most.

Fig.~\ref{fig:teaser} collects the three obstacles and the component of
\method{} that addresses each. We make four contributions. (1) We introduce \emph{RPC ray-field
observability} as a measurable property of feed-forward geometry features and
show how lightweight adaptation restores localization along native RPC height
rays. (2) We formulate the low-order \emph{height--datum gauge} of RPC
reconstruction, quantify its identifiability on delivered cameras, and convert
it into an explicit datum mechanism that supports multiple control regimes.
(3) We derive and test a \emph{conditional gain law} linking the value of
multi-view evidence to predicted monocular uncertainty under calibrated
fusion. (4) We introduce \bench{}, a unified absolute-frame evaluation of
eighteen systems across three generalization tiers with explicit leakage and
completeness controls.

\section{Related Work}
\label{sec:related}

\subsection{Feed-Forward Geometry Foundation Models}
Feed-forward 3D models regress dense geometry directly from images, from
pointmap-based multi-view systems~\cite{wang2024dust3r,leroy2024mast3r,wang2025vggt,keetha2026mapanything}
to monocular metric or affine-invariant priors
~\cite{yang2024depthanything,yang2024depthanything2,lin2025da3,yin2023metric3d,hu2024metric3dv2,wang2025moge2}.
Their representations build on large vision transformers
~\cite{dosovitskiy2021vit,oquab2024dinov2,ranftl2022midas}, but their geometric
parameterizations are developed for central-perspective imagery. Applying
these models to pushbroom data therefore raises a representation question:
are their features still localizable when correspondence is sampled along
RPC object-space height rays? EO-VGGT~\cite{luo2026eovggt} independently
shows the value of injecting orbital geometry into a frozen VGGT backbone.
\method{} takes up the adjacent questions of ray-field observability,
absolute elevation, and the conditional value of multi-view evidence.

\subsection{Satellite 3D Reconstruction under RPC Cameras}
Classical satellite photogrammetry operates directly on RPC geometry
~\cite{tao2001rfm,gupta1997pushbroom}. S2P~\cite{defranchis2014s2p} and
related stereo pipelines combine hand-crafted matching
~\cite{hirschmuller2008sgm,facciolo2015mgm}, while generic reconstruction pipelines such as COLMAP~\cite{schonberger2016colmap}
can be used through local perspective approximations~\cite{zhang2019vissat}.
Learned satellite MVS replaces hand-crafted matching with differentiable RPC
warping and cost volumes~\cite{yao2018mvsnet,gu2020cascade,gao2021satmvs,gao2023satmvsf},
but typically relies on satellite-specific training and dense height
supervision. Neural rendering provides a second line of work:
S-NeRF, Sat-NeRF, and EO-NeRF~\cite{derksen2021snerf,mari2022satnerf,mari2023eonerf}
optimize each scene under satellite cameras in the spirit of instant
neural fields~\cite{mueller2022instant}; EOGS and EOGS++
~\cite{aira2025eogs,bournez2025eogspp} reduce that cost with Gaussian
splatting~\cite{kerbl2023gaussian}, and
SatSplat~\cite{satsplat2026} adds online camera refinement; and
RPC-GS~\cite{wagner2026rpcgs} renders Gaussians natively through RPCs. These
methods establish strong per-scene geometry, but they do not resolve the
feed-forward question addressed here: what native multi-view RPC evidence can
identify about absolute height and how that information should interact with
a reusable foundation prior.

\subsection{Generalizable Sparse-View Reconstruction}
Generalizable Gaussian splatting predicts scene primitives in one forward
pass. pixelSplat~\cite{charatan2024pixelsplat}, MVSplat
~\cite{chen2024mvsplat}, DepthSplat~\cite{xu2025depthsplat}, HiSplat
~\cite{tang2025hisplat}, TranSplat~\cite{zhang2025transplat}, and
AnySplat~\cite{jiang2025anysplat} progressively improve sparse-view
reconstruction through learned matching, monocular priors, and hierarchical or
transformer aggregation. Satellite variants such as SkySplat
~\cite{huang2026skysplat} and SatSurfGS~\cite{chen2026satsurfgs} demonstrate
that generalizable splatting can be adapted to multi-temporal orbital imagery.
\method{} differs in both output and analysis. Its primary prediction is a
geodetically referenced height field, and it keeps observable relief, datum
freedom, and uncertainty-weighted evidence apart instead of folding their
interaction into a single learned fusion block.

\subsection{Uncertainty and Gauge Ambiguity}
Heteroscedastic regression~\cite{kendall2017uncertainties}, deep ensembles
~\cite{lakshminarayanan2017ensembles}, and post-hoc calibration
~\cite{guo2017calibration} provide established tools for predictive
uncertainty. Photogrammetric adjustment, in parallel, has long recognized the
coupling between absolute geolocation and systematic sensor error: vendor RPCs
are corrected by low-order image-space bias models
~\cite{grodecki2003block,fraser2003bias}, and bundle adjustment is defined only
up to gauge freedoms fixed by external constraints~\cite{triggs2000bundle}.
A remaining gap is a unified feed-forward treatment in which this ambiguity is
made explicit, learned uncertainty is calibrated on the same output frame, and
sparse control has a precise role. \method{} connects these ideas in a single
satellite reconstruction model.

\begin{figure*}[!t]
    \centering
    \includegraphics[width=\textwidth]{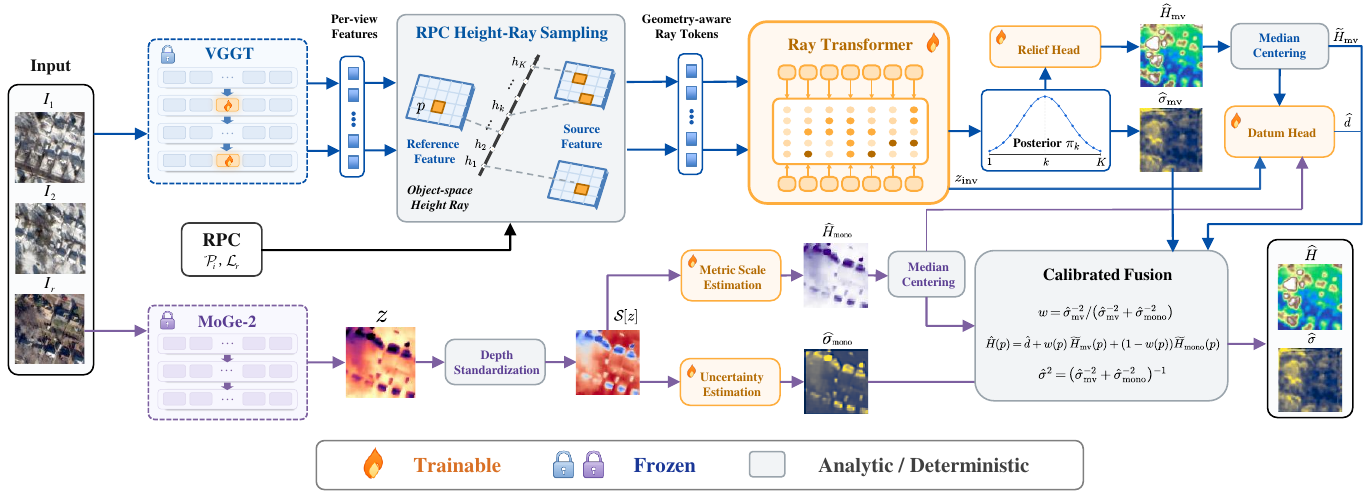}
    \caption{\textbf{GeoRay architecture.} RPC ray reasoning estimates relief,
the datum head anchors absolute level to the RPC height coordinates, and
calibrated precision fusion returns $\widehat H$ with uncertainty
$\widehat\sigma$. Lock symbols mark frozen backbones, flame symbols mark
trainable modules, and gray blocks denote parameter-free analytic operators.}
    \label{fig:pipeline}
\end{figure*}

\section{Method}
	\label{sec:method}
	
	\subsection{Problem Setting: Reconstruction in RPC Object Space}
	\label{sec:problem}
	
	A satellite acquisition provides views
	$\{(I_i,\mathcal{P}_i)\}_{i=1}^{N}$ of a geographic tile, where
	$\mathcal{P}_i$ is the rational polynomial camera of view $i$: a
	pair of rational maps between normalized geodetic coordinates
	$(\phi,\lambda,h)$ and normalized image coordinates
	$(r,c)$~\cite{tao2001rfm}. Because the pushbroom sensor assigns
	each scanline its own perspective center~\cite{gupta1997pushbroom},
	no single global pinhole epipolar geometry exists and depth along a pinhole ray is
	not a meaningful variable. The natural parameterization is instead
	the \emph{localization ray}: for a pixel $p$ in a reference view
	$r$ and a height hypothesis $h$,
	\begin{equation}
		\mathbf{x} = \mathcal{L}_r(p, h) \in \mathbb{R}^3,
		\qquad
		q_i = \mathcal{P}_i(\mathcal{L}_r(p, h)),
		\label{eq:ray}
	\end{equation}
	where $\mathcal{L}_r$ is the inverse RPC of the reference view and
	$q_i$ the reprojection into view $i$. Sweeping $h$ traces a curve
	in object space, sampled in \emph{meters of geodetic height} rather
	than in depth. The sensitivity of correspondence to height is the
	\emph{disparity rate}
	\begin{equation}
		\rho_i(p) = \left\| \tfrac{\partial q_i}{\partial h}(p, h)
		\right\|_2 \quad [\text{px/m}],
		\label{eq:disprate}
	\end{equation}
	which is nearly constant over the relevant bracket and is set by
	the convergence geometry of the acquisition; it plays the role that
	the baseline-over-depth ratio plays for pinhole stereo, and it will
	recur as the natural conditioning variable for both adaptation and
	fusion. Given $N$ views ($N{=}3$ throughout), the task is to
	predict a geodetically referenced height field
	$\widehat{H}:\Omega\to\mathbb{R}$ with per-pixel uncertainty
	$\widehat{\sigma}$ in a single forward pass, evaluated in the
	absolute vertical datum without any alignment to reference data.
	Vendor RPCs carry systematic pointing error, classically corrected
	by low-order bias
	compensation~\cite{grodecki2003block,fraser2003bias}; the coupling
	this induces between absolute height and sensor bias is formalized
	in Sec.~\ref{sec:gauge}.
	
	\subsection{Architecture Overview}
\label{sec:overview}

Fig.~\ref{fig:pipeline} summarizes the model. Multi-view images are encoded by
a frozen VGGT backbone~\cite{wang2025vggt}; trainable low-rank adapters make
its features suitable for RPC ray sampling. For each reference pixel, RPC
geometry maps a set of metric height hypotheses into the source feature maps,
forming geometry-aware ray tokens that are reasoned over by a compact Ray
Transformer. The resulting height posterior yields a multi-view relief
estimate $\widehat{H}_{\mathrm{mv}}$ and uncertainty
$\widehat{\sigma}_{\mathrm{mv}}$. In parallel, frozen MoGe-2
~\cite{wang2025moge2} is queried once on the reference image and converted to
a metric monocular relief estimate $\widehat{H}_{\mathrm{mono}}$ with
uncertainty $\widehat{\sigma}_{\mathrm{mono}}$. The datum head anchors the
absolute level, which ray matching alone cannot identify, to the height
coordinates delivered with the RPCs and corrects it with a learned residual;
calibrated fusion then combines the median-centered relief fields according to their predicted
precision.

The complete prediction is
\begin{equation}
(\widehat{H}, \widehat{\sigma})
= \mathcal{F}\!\left(\{I_i,\mathcal{P}_i\}_{i=1}^{N};\,
\theta_{\mathrm{adapt}}\right),
\label{eq:output}
\end{equation}
where $\theta_{\mathrm{adapt}}$ contains only low-rank adapters and
lightweight heads: $12.4$\,M trainable parameters versus $1.240$\,B frozen
parameters, approximately one percent of the system; both foundation backbones
remain frozen. Each obstacle has its own mechanism: RPC observability, datum
identifiability, and heterogeneous-evidence fusion.

\subsection{RPC Ray-Field Observability}
	\label{sec:observability}
	
	\subsubsection{The property}
	Let $F_i = \Phi(I_i) \in \mathbb{R}^{H'\times W'\times C}$ be
	features from a frozen encoder $\Phi$. For reference pixel $p$ and
	hypothesis heights $\{h_k\}_{k=1}^{K}$ spanning a bracket
	$[h_{\min}, h_{\max}]$, define ray-field samples and tokens
	\begin{align}
		f_{i,k}(p) &= F_i\!\left(\mathcal{P}_i(\mathcal{L}_r(p,
		h_k))\right),
		\label{eq:sampling}\\
		t_k(p) &= \mathcal{A}\!\left(f_{1,k}(p), \dots,
		f_{N,k}(p)\right),
		\label{eq:token}
	\end{align}
	where $\mathcal{A}$ aggregates across views (feature correlation
	with a learned mixing layer). A matching distribution follows from
	a scoring head $s$:
	\begin{equation}
		\pi_k(p) = \frac{\exp\!\big(s(t_k(p))\big)}
		{\sum_{k'}\exp\!\big(s(t_{k'}(p))\big)}.
		\label{eq:dist}
	\end{equation}
	We call the pair $(\Phi, s)$ \emph{ray-field observable} on a
	domain if $\pi(p)$ concentrates near the true surface height for
	textured, unoccluded pixels. The definition is operational: concentration (the entropy of
		$\pi$), localization (the distance from the mode to the true
		height), and the dependence of both on the disparity rate $\rho$
		can each be measured directly.
	Perspective pretraining does not optimize features for comparison at
	object-space points reached by rational maps under scanline-varying
	projection centers. Sec.~\ref{sec:exp-observability} shows that frozen profiles can be
		locally sharp yet systematically mislocalized. What adaptation has
		to restore is localization, for which raw peak contrast is not a
		proxy.
	
	\subsubsection{Ray-consistent adaptation}
	We restore observability with low-rank adapters~\cite{hu2022lora}
	$\Delta\Phi$ on the frozen encoder, trained specifically through
	ray-field supervision, and a compact geometry token attached to each
	height hypothesis before scoring. The implemented token records the
	hypothesis coordinate \emph{relative to the bracket} and its step,
	together with disparity-rate--step, pairwise viewing angle, and
	ground-sampling-distance cues; expressed this way, every entry of the
	token is invariant to the change of vertical origin introduced in
	Sec.~\ref{sec:gauge}. The same scoring head can then condition its evidence on the local
		acquisition geometry without replacing the pretrained representation
		or fitting a pinhole surrogate. The expected height $\sum_k \pi_k h_k$ is the height-domain
		counterpart of the soft-argmin readout of learned
		stereo~\cite{kendall2017gcnet}. A lightweight \emph{relief head}, a
		convolutional decoder over the token field, refines it by sharpening
		discontinuities and suppressing per-pixel outliers. Since the head
		reads only ray-field internals, its output inherits the invariance of
		Remark~\ref{rem:blindness}: it refines relief and leaves the datum
		untouched. The multi-view branch outputs $\widehat{H}_{\mathrm{mv}}$ and
	$\widehat{\sigma}_{\mathrm{mv}}$. Sec.~\ref{sec:exp-observability}
	isolates the representation-level effect of the adaptation and tests its
	dependence on the supplied RPC geometry.
	
	\subsection{The Height--Datum Gauge}
	\label{sec:gauge}
	
	\subsubsection{Observation model and identifiability}
	Sensor pointing error is classically modeled as a low-order
	image-space correction to each RPC~\cite{grodecki2003block,
		fraser2003bias}. Writing $\beta_i$ for the (unknown) correction
	parameters of view $i$, the reprojection consistent with a height
	field $H$ is
	\begin{equation}
		q_i(p) = \mathcal{P}_i\!\left(\mathcal{L}_r(p, H(p))\right)
		+ A_i(p)\,\beta_i,
		\label{eq:obsmodel}
	\end{equation}
	with $A_i$ the design matrix of the bias model (constant and linear
	image-coordinate terms). Multi-view matching observes only the
	$q_i$. Linearizing at a working solution, write
	$\mathbf{j}_i(p) = \partial q_i / \partial h \in \mathbb{R}^{2}$ for
	the height Jacobian of view $i$, so that $\rho_i = \|\mathbf{j}_i\|_2$
	of Eq.~\eqref{eq:disprate} is its magnitude. A vertical perturbation
	$\delta H$ displaces view $i$'s reprojection by
	$\mathbf{j}_i(p)\,\delta H(p)$, and is unobservable exactly when that
	image-space footprint is reproducible by admissible bias updates:
	\begin{equation}
		\mathbf{j}_i(p)\,\delta H(p) + A_i(p)\,\Delta\beta_i
		= \mathbf{0} \quad \forall i,p.
		\label{eq:lin}
	\end{equation}
	
	\begin{proposition}[First-order gauge and gauge fixing]
\label{prop:gauge}
Fix a tile, and let $\varepsilon>0$ denote the reprojection precision below
which matching evidence is not informative. Define
\begin{equation}
R(\delta H)=\min_{\{\Delta\beta_i\}}\max_i
\big\|\mathbf{j}_i\,\delta H+A_i\,\Delta\beta_i\big\|_{L^2(\Omega)},
\label{eq:gaugeresid}
\end{equation}
where the norm is the root mean square over tile pixels. Write
$\mathbf{j}_i(p)=\rho_i(p)\mathbf{n}_i(p)$ and let $\eta$ bound the
relative spatial variation of $\rho_i$ and $\mathbf{n}_i$ over the tile.
For the affine height family
$\mathcal{B}=\operatorname{span}\{1,x,y\}$ induced by the constant and
linear terms of the standard image-space bias model:
(i) every $b\in\mathcal{B}$ admits an admissible bias update with
$R(b)=O(\eta\,\rho\,\|b\|)$, and is therefore $\varepsilon$-unobservable
whenever this residual lies below $\varepsilon$;
(ii) in the locally affine, constant-Jacobian limit $\eta\to0$,
$\mathcal{B}$ is an exact three-dimensional null subspace of the
linearized observation model; and
(iii) control points $\{(x_j,h_j^*)\}_{j=1}^k$ fix any chosen components
of this subspace through
\begin{equation}
\widehat H^{(k)}=\widehat H+b^*,\quad
b^*=\argmin_{b\in\mathcal B}\sum_{j=1}^k
\big(\widehat H(x_j)+b(x_j)-h_j^*\big)^2,
\label{eq:gcp}
\end{equation}
which is unique whenever the corresponding control-point design has full
column rank. In particular, one point fixes the scalar sub-gauge
$\operatorname{span}\{1\}$. The accuracy of the affine approximation is a
property of the cameras, not of a reconstruction model, and is measured
directly in Sec.~\ref{sec:exp-gauge}.
\end{proposition}

\begin{proof}
At first order, a height perturbation is hidden whenever its image-space
footprint $\mathbf{j}_i\delta H$ can be represented by the admissible bias
field $A_i\Delta\beta_i$. In the constant-Jacobian limit, the constant and
linear bias terms generate an affine scalar field along the fixed height
sensitivity direction, so every $b\in\mathcal B$ has a bias witness and
therefore zero residual. Allowing $\rho_i$ and $\mathbf n_i$ to vary by a
relative $O(\eta)$ while retaining that witness perturbs the cancellation
by $O(\eta\rho_i|b|)$ pointwise, yielding the stated RMS bound. The control
problem in Eq.~\eqref{eq:gcp} is ordinary least squares on the selected
gauge basis and is unique under the stated rank condition. The proposition
	establishes this modeled low-order null subspace; it does not require that no
	additional near-null modes exist for a particular camera configuration.
\end{proof}
	
		\begin{figure*}[!tb]
		\centering
		\includegraphics[width=\textwidth]{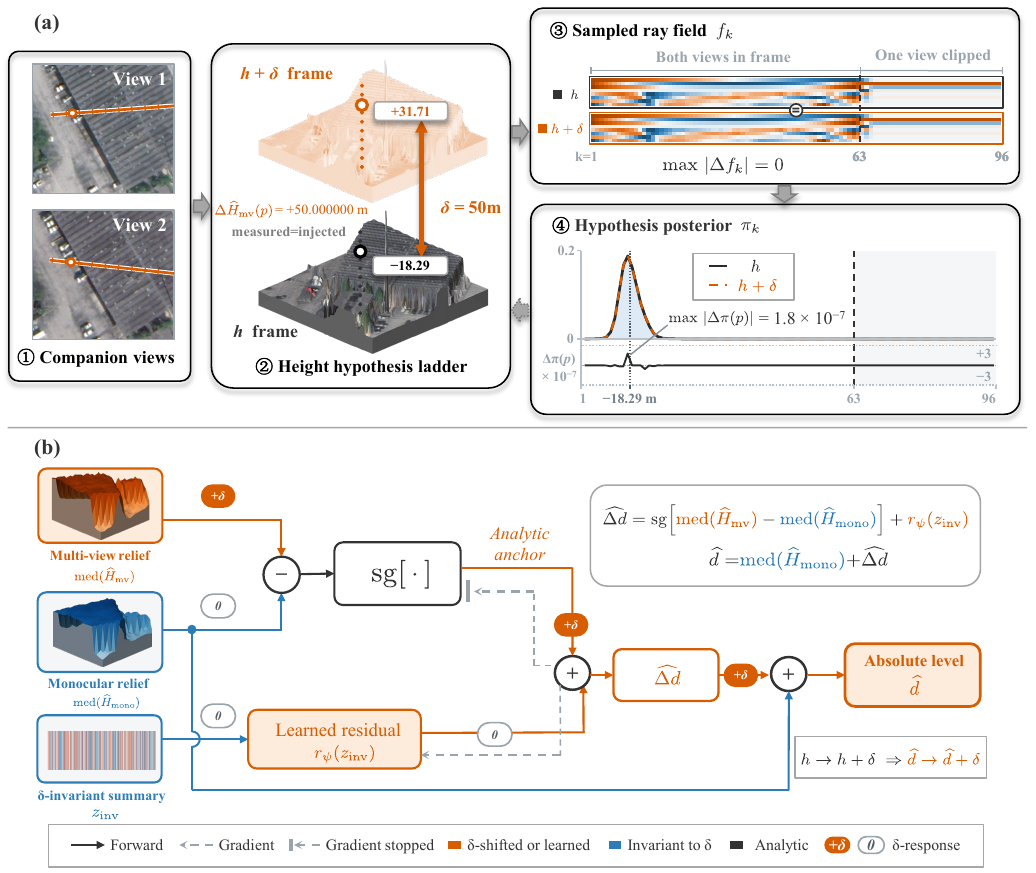}
		\caption{\textbf{Height--datum gauge and datum recovery.}
			(a) A joint vertical-origin shift leaves RPC ray samples and the posterior
			unchanged while translating the height hypotheses, demonstrating that
			ray-internal evidence cannot identify the datum. (b) The datum head combines
			an equivariant analytic anchor with a $\delta$-invariant residual to place the
				absolute level.}
		\label{fig:gauge-mechanism}
	\end{figure*}

	The experiments probe the scalar sub-gauge with one control point and the
	full affine family with ten spatially distributed points. The same correction
	basis is applied to every method, and Sec.~\ref{sec:exp-gauge} measures how
	accurately the first-order gauge describes the delivered RPCs.
	
	\subsubsection{An exact coordinate-origin symmetry}
	The scalar gauge admits an exact training-time construction. Let
	$\delta \in \mathbb{R}$ and transform a training sample by
	\begin{equation}
		H^{*} \mapsto H^{*} + \delta, \quad
		o_i \mapsto o_i + \delta \quad \forall i,
		\label{eq:shift}
	\end{equation}
	where $o_i$ is the height offset of view $i$'s RPC; the hypothesis
	bracket, anchored to the reference offset, shifts with it, and the
	images are untouched. This is a change of vertical origin
	rather than an image-space bias step: the RPC evaluates the
	normalized height $(h - o_i)/s_i$, so shifting $H^{*}$ and $o_i$
	together leaves that argument, and hence every reprojection,
	\emph{identically} unchanged, with no bias update required. The
	symmetry is therefore exact, and distinct from the first-order gauge
	of Proposition~\ref{prop:gauge}, which concerns height perturbations
	absorbed by bias updates at fixed cameras. The transformed sample is exactly
as consistent with the observations as the original: the entire ray field is
unchanged, while object-space samples translate vertically with the frame.
Because the two samples are observationally identical,
Eq.~\eqref{eq:shift} provides no learnable cue for the datum estimator;
instead, it defines an equivariance requirement: the estimate must move with
$\delta$ at unit slope although the ray field is unchanged. The shift serves three purposes: a diagnostic that exposes which pathways
carry the origin (Remark~\ref{rem:blindness},
Fig.~\ref{fig:gauge-mechanism}(a)), a constraint on where a datum estimator
may look, and a held-out check on the resulting construction
(Sec.~\ref{sec:exp-gauge}).

	\begin{remark}[Gauge blindness of ray-internal readouts]
		\label{rem:blindness}
		Under the joint shift of Eq.~\eqref{eq:shift}, every reprojection
		$q_{i,k}$, every sampled feature, and hence every quantity
		computed from the ray field alone, including its pooled summary
		statistics, is invariant to $\delta$. Any
		datum estimator that reads only ray-field internals is
		structurally blind to the gauge it must track; equivariance is
		impossible for it, not merely difficult to learn.
	\end{remark}
	
	Remark~\ref{rem:blindness} is verified end to end in
	Fig.~\ref{fig:gauge-mechanism}(a) and dictates the design of
	Fig.~\ref{fig:gauge-mechanism}(b). In the multi-view branch the hypothesis coordinates $h_k$ are the only
quantity that carries the vertical origin, so the tile level of its readout,
$\med(\widehat H_{\mathrm{mv}})$, is exactly equivariant to
Eq.~\eqref{eq:shift}, while the monocular relief of Sec.~\ref{sec:fusion} is
exactly invariant. Their difference is thus a $\delta$-sensitive scalar available in closed
form. We use it as an analytic anchor and confine
learnable capacity to a residual that is $\delta$-invariant by construction:
	\begin{equation}
		\widehat{\Delta d} = \underbrace{\sg\!\left[\med\big(\widehat{H}_{\mathrm{mv}}\big)
            - \med\big(\widehat{H}_{\mathrm{mono}}\big)\right]}_{
            \text{analytic anchor (equivariant)}}
		\; + \;
		r_{\psi}\big(z_{\mathrm{inv}}\big),
		\label{eq:datumhead}
	\end{equation}
where $\med(\cdot)$ is the tile median, $\sg[\cdot]$ is stop-gradient,
$z_{\mathrm{inv}}$ is a pooled $\delta$-invariant ray summary, and
$r_{\psi}$ is a lightweight residual MLP. The learned residual is isolated from the analytic anchor and cannot alter
its coordinate-origin response, so Eq.~\eqref{eq:datumhead} is exactly
equivariant to the shift in Eq.~\eqref{eq:shift}. The final absolute level is
\begin{equation}
    \widehat{d}=\med\big(\widehat{H}_{\mathrm{mono}}\big)
    +\widehat{\Delta d},
    \label{eq:datumlevel}
\end{equation}
Because the monocular median cancels algebraically in the composition of
Eqs.~\eqref{eq:datumhead}--\eqref{eq:datumlevel}, the forward pass reduces
to $\widehat d=\med(\widehat H_{\mathrm{mv}})+r_{\psi}(z_{\mathrm{inv}})$:
the absolute level is the tile level of the ray-field readout plus a learned
correction that can only depend on invariant ray statistics. We keep the
monocular-referenced form because it is the form in which
Fig.~\ref{fig:gauge-mechanism}(b) and the anchor loss of
Sec.~\ref{sec:objectives} are stated, and because it makes explicit that
the monocular branch contributes relief only. The analytic term fixes the gauge response; the residual corrects systematic
level error without touching it. In zero-control mode the vertical origin is the one supplied by the delivered
RPC height coordinates, and the residual
corrects the systematic part of the vendor-level error that is predictable
from invariant ray statistics; a bias-induced shift in the affine family of
Proposition~\ref{prop:gauge} that is not systematic remains unidentifiable
from ray evidence alone and must be resolved by external control.
The same trained model supports
zero control, a single scalar-height control, or sparse affine control through
Eq.~\eqref{eq:gcp}.

	\subsection{Calibrated Fusion and the Conditional Gain Law}
	\label{sec:fusion}
	
	\subsubsection{From the monocular prior to metric relief}
The monocular branch is used as a relief prior, not as an absolute-height
estimator. MoGe-2~\cite{wang2025moge2} is queried once on the reference
image. Its scalar depth field $z$ is passed through the deterministic
standardization used by its public implementation, denoted
$\mathcal{S}[z]$, and a lightweight learned scale maps that standardized
field to metric relief,
\begin{equation}
    \widehat{H}_{\mathrm{mono}}(p)=a_{\theta}\,\mathcal{S}[z](p),
    \label{eq:monoconv}
\end{equation}
where $\mathcal{S}[z]$ is centered and normalized within the tile. No
geodetic offset is taken from the monocular model; the absolute level is
owned exclusively by the datum mechanism of
Sec.~\ref{sec:gauge}. This separation is also what makes the monocular
relief invariant to the coordinate-origin shift of
Eq.~\eqref{eq:shift}.

\subsubsection{Placing two streams on one uncertainty scale}
The monocular calibration pathway predicts the metric relief scale and a
per-pixel uncertainty $\widehat{\sigma}_{\mathrm{mono}}$; its height loss
is evaluated after median centering, so it cannot absorb the absolute
datum. The multi-view pathway supplies its own per-pixel relief uncertainty
$\widehat{\sigma}_{\mathrm{mv}}$ alongside the ray-field estimate. Both are
trained with the heteroscedastic likelihood term in
Eq.~\eqref{eq:total}, and Sec.~\ref{sec:exp-gain} evaluates both the
ordering and the empirical interval scale before these uncertainties are
used for fusion.

\subsubsection{Fusion}
Fusion is performed in relief space. Let
$\widetilde{H}=H-\med(H)$ and define the multi-view precision weight
\[
w(p)=\frac{\widehat{\sigma}^{-2}_{\mathrm{mv}}(p)}
{\widehat{\sigma}^{-2}_{\mathrm{mv}}(p)+
 \widehat{\sigma}^{-2}_{\mathrm{mono}}(p)}.
\]
The absolute prediction is
\begin{equation}
\widehat{H}(p)=\widehat{d}+w(p)\widetilde{H}_{\mathrm{mv}}(p)
+\big(1-w(p)\big)\widetilde{H}_{\mathrm{mono}}(p),
\label{eq:fusion}
\end{equation}
with
$\widehat{\sigma}^{2}(p)=\big(\widehat{\sigma}^{-2}_{\mathrm{mv}}(p)+
\widehat{\sigma}^{-2}_{\mathrm{mono}}(p)\big)^{-1}$. The rule is minimal by design, which keeps the arbitration transparent: each
source contributes in proportion to its predicted precision. The conditional-independence approximation between the two streams can make
the fused uncertainty optimistic, so Sec.~\ref{sec:exp-gain} evaluates
interval coverage in both the relief and the absolute frame.

\subsubsection{The conditional gain law}
Define the pointwise gain over the calibrated monocular prior as
\begin{equation}
G(p)=\big|\widehat{d}+\widetilde{H}_{\mathrm{mono}}(p)-H^{*}(p)\big|
-\big|\widehat{H}(p)-H^{*}(p)\big|,
\label{eq:gain}
\end{equation}
which is positive where multi-view evidence improves the prediction. The
\emph{conditional gain law} states that
$\mathbb{E}[G\mid\widehat{\sigma}_{\mathrm{mono}}]$ increases with predicted
monocular uncertainty: geometry should contribute most where the prior itself
signals unreliability. The conditioning variable is available before the reconstruction error is
known, so the relation is prospective, not a post-hoc error taxonomy. Sec.~\ref{sec:exp-gain} tests the association against
a pairing-permutation null and verifies the calibration required to interpret
$\widehat{\sigma}_{\mathrm{mono}}$ as a reliability variable.

\subsection{Training Objective}
\label{sec:objectives}

All trainable components are optimized jointly:
\begin{equation}
\begin{aligned}
\mathcal{L}={}&\lambda_{\mathrm h}\mathcal L_{\mathrm h}
+\lambda_{\mathrm{ray}}\mathcal L_{\mathrm{ray}}
+\lambda_{\mathrm{mono}}\mathcal L_{\mathrm{mono}} \\
&+\lambda_{\mathrm{anc}}\mathcal L_{\mathrm{anc}}
+\lambda_{\mathrm{nll}}\mathcal L_{\mathrm{nll}}
+\lambda_{\mathrm{grad}}\mathcal L_{\mathrm{grad}}.
\end{aligned}
\label{eq:total}
\end{equation}
The terms have distinct structural roles. $\mathcal L_{\mathrm h}$ supervises
fused height, $\mathcal L_{\mathrm{ray}}$ supervises the height posterior in
Eq.~\eqref{eq:dist}, $\mathcal L_{\mathrm{mono}}$ trains the monocular metric
relief after median centering, $\mathcal L_{\mathrm{nll}}$ calibrates the two
predicted variances, and $\mathcal L_{\mathrm{grad}}$ preserves height
discontinuities. The anchor loss trains the datum residual on the tile level,
\begin{equation}
\mathcal L_{\mathrm{anc}}=\left|\widehat{\Delta d}-
\sg\!\left(\med(H^*)-\med(\widehat H_{\mathrm{mono}})\right)\right|,
\label{eq:anc}
\end{equation}
which, with the monocular medians cancelling, makes $r_{\psi}$ predict the
residual tile-level error of the multi-view readout from invariant ray
statistics. All loss weights are fixed across the reported training runs.
Two identities hold by construction under the shift of Eq.~\eqref{eq:shift};
we monitor them instead of optimizing them,
\begin{align}
\mathcal C_{\mathrm{trk}}&=\left|\left(\widehat{\Delta d}^{\,\delta}
-\widehat{\Delta d}\right)-\delta\right|,\\
\mathcal C_{\mathrm{inv}}&=\left\|\left(\widehat H^{\delta}-\delta\right)
-\widehat H\right\|_1;
\end{align}
both vanish identically for the construction of Eq.~\eqref{eq:datumhead},
so they contribute no gradient and serve as consistency assertions that
would expose any leak of the vertical origin into a nominally invariant
pathway. Sec.~\ref{sec:exp-gauge} reports their held-out value.

\section{\bench{}: Absolute-Frame Evaluation}
	\label{sec:benchmark}
	
	Existing evaluation conventions can obscure absolute satellite 3D: registration
removes the vertical placement under study, test-reference height ranges leak
information into learned predictors, and pointwise error alone can reward
partial surfaces. \bench{} evaluates public datasets under two governing principles:
predictions are scored first in the absolute geodetic frame, and all test-time
information comes from imagery, RPC metadata, or training-side statistics,
never from the evaluated reference.
	
	\begin{figure*}[!tb]
		\centering
		\includegraphics[width=\textwidth]{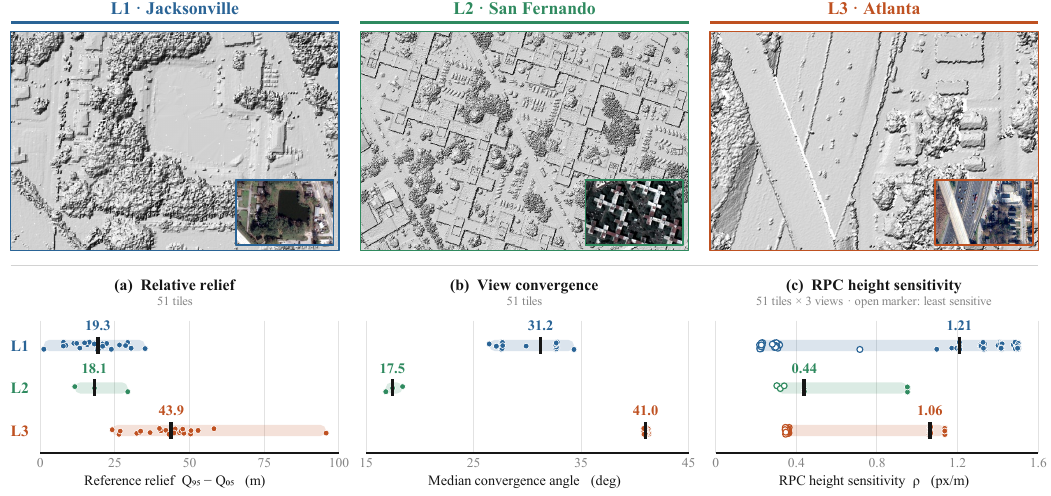}
		\caption{\textbf{Domain shifts across the three \bench{} tiers.}
Reference relief, view convergence, and RPC height sensitivity summarize the
complementary geometric shifts from L1 to L3.}
		\label{fig:benchmark}
	\end{figure*}

\subsection{Three Tiers of Generalization}
	\label{sec:tiers}
	
	\begin{table}[!tb]
		\centering
		\caption{Three \bench{} generalization tiers and their data sources.}
		\label{tab:tiers}
		\footnotesize
		\setlength{\tabcolsep}{3pt}
		\begin{tabular}{@{}llll@{}}
			\toprule
			& \textbf{L1 In-domain} & \textbf{L2 Cross-dataset}
			& \textbf{L3 Cross-city} \\
			\midrule
			Source & US3D~\cite{bosch2019us3d} & IARPA
			MVS3D~\cite{bosch2016mvs3d} & ATL-SN4~\cite{weir2019spacenet} \\
			Sensor & WorldView-3 & WorldView-3 & WorldView-2 \\
			Region & Jacksonville, Omaha & San Fernando & Atlanta \\
			Units & 234 / 26 tiles & 3 AOIs & 22 tiles \\
			Reference & US3D DSM & airborne lidar & SN4 DSM \\
			Shift & none (train dist.) & region\,+\,reference &
			sensor\,+\,city\,+\,angle \\
			\bottomrule
		\end{tabular}
	\end{table}
	
	Table~\ref{tab:tiers} separates three forms of transfer.
\textbf{L1} uses spatially disjoint US3D~\cite{christie2020geocentric,lesaux2019dfc} tiles from Jacksonville and Omaha.
\textbf{L2} evaluates the unchanged model on IARPA MVS3D, introducing a new
region and an independent airborne-lidar reference. \textbf{L3} changes sensor
generation, city, and viewing geometry simultaneously using SpaceNet-4 Atlanta.
We report the tiers separately because they probe different failure axes;
Fig.~\ref{fig:benchmark} summarizes the associated relief and acquisition
geometry.

	\subsection{Evaluation Principles}
\label{sec:contract}

Four rules define the evaluation; Table~\ref{tab:contract} states them in
operational form.

\begin{table}[!tb]
		\centering
		\caption{The \bench{} evaluation contract.}
		\label{tab:contract}
		\footnotesize
		\setlength{\tabcolsep}{5pt}
		\begin{tabular}{@{}p{2.35cm}p{5.55cm}@{}}
			\toprule
			\textbf{Requirement} & \textbf{Operational rule} \\
			\midrule
			Absolute scoring & No registration before scoring; aligned
			metrics reported as diagnostics only. \\
			\addlinespace[2pt]
			Control budgets & Every row labeled $k \in \{0, 1, 10\}$;
			identical correction family for all methods. \\
			\addlinespace[2pt]
			Leakage control & Test-time inputs enumerated per method;
			reference-derived tile quantities excluded. \\
			\addlinespace[2pt]
			Visible degeneracy & Accuracy, coverage, and relief fidelity
			reported jointly; frozen family-level output adapters. \\
			\bottomrule
		\end{tabular}
	\end{table}
 \emph{Absolute placement} is ranked before
any registration; aligned error is reported only to diagnose removable
translation. \emph{Control is explicit}: every result states whether it uses
zero, one, or ten height controls, with the same correction family available
to all methods. \emph{Test information is leakage-free}: hypothesis ranges
and normalizations are derived from RPC metadata or training statistics, not
from the evaluated reference. \emph{Completeness is visible}: absolute error
is reported together with coverage, relief fidelity, and
$\mathrm{PAG}^{\mathrm{c}}$, so partial or flattened surfaces cannot appear
strong through a favorable valid-pixel subset. Per-scene optimization is reported separately, as an accuracy--compute
reference, and is kept out of the feed-forward ranking.

	\section{Experiments}
	\label{sec:experiments}
	
	\subsection{Experimental Setup}
	\label{sec:setup}
	
	\subsubsection{Metrics}
Let $\Omega_{\mathrm{val}}$ denote reference-valid pixels and
$\Omega_{\mathrm{eval}}\subseteq\Omega_{\mathrm{val}}$ the pixels predicted by
a method. \emph{Absolute MAE} is measured directly in the geodetic frame;
\emph{aligned MAE} is a diagnostic after bounded 3-DoF translation, with the
accepted-registration fraction reported as Reg. To couple accuracy with
completeness, we use
\begin{equation}
\mathrm{PAG}^{\mathrm{c}}_{\tau}=\frac{\big|\{p\in\Omega_{\mathrm{eval}}:
|\widehat{H}_{\mathrm{reg}}(p)-H^{*}(p)|<\tau\}\big|}
{|\Omega_{\mathrm{val}}|},
\label{eq:pag}
\end{equation}
with $\tau\in\{2.5,7.5\}$\,m; unpredicted pixels count as failures.
Here $\widehat H_{\mathrm{reg}}$ is the prediction after the diagnostic
translation, which is computed for every tile; the acceptance guard decides
only which tiles enter the Aln.\ average (the fraction reported as Reg.), so
$\mathrm{PAG}^{\mathrm c}_{\tau}$ is defined for every method regardless of
Reg.
Coverage is reported using the denominator native to each benchmark tier, and
\emph{relief fidelity} is the ratio between the predicted and reference
$5$--$95$ percentile height spreads. We also report error on strong
height discontinuities. For \method{}, uncertainty quality is measured by
Gaussian negative log-likelihood, interval coverage, and AUSE from
$\widehat{\sigma}$-ordered sparsification~\cite{ilg2018uncertainty}. All
aggregates are unweighted means over tiles.
Table~\ref{tab:evalmatrix} maps every experiment to its data support,
training budget, and comparison track.

\begin{table}[!tb]
		\centering
		\caption{Mapping from each experiment to its data support,
			track, training budget, and reporting location.}
		\label{tab:evalmatrix}
		\footnotesize
		\setlength{\tabcolsep}{2.5pt}
		\begin{tabular}{@{}lllll@{}}
			\toprule
			\textbf{Experiment} & \textbf{Support} & \textbf{Track} &
			\textbf{Budget} & \textbf{Where} \\
			\midrule
			Main comparison & L1: 26 tiles & FF-MV & prod. &
			Sec.~\ref{sec:exp-main} \\
			Generalization & L2: 3 AOIs; L3: 22 & FF-MV &
			prod. & Sec.~\ref{sec:exp-generalization} \\
			Per-scene context & 6 L1 tiles; 3 AOIs & per-scene\,\dag &
			prod. & Sec.~\ref{sec:exp-perscene} \\
			Mechanisms I--III & L1 (+L2 where noted) & --- & prod. &
			V-E--V-G \\
			Ablations & L1: 26 tiles & FF-MV & red. &
			Sec.~\ref{sec:exp-ablation} \\
			Backbone swaps & L1: 26 tiles & FF-MV & red. &
			Sec.~\ref{sec:exp-ablation} \\
			\bottomrule
		\end{tabular}
	\end{table}

\subsubsection{Geodetic reference handling}
All heights are expressed in the height convention of the reference product
released with each tier, and RPCs are used as delivered, with the hypothesis
bracket anchored to their height offset. \emph{Absolute} here means scored in that frame with no alignment to the
reference; it is not a claim of independently surveyed orthometric accuracy. Any residual difference between
the RPC and reference height conventions would enter as a tile-constant
offset, i.e., the scalar sub-gauge $\operatorname{span}\{1\}$ that the
$k{=}1$ experiments of Sec.~\ref{sec:exp-gauge} measure directly; the
anchor-only absolute error on L1 ($\AnchorAbs$\,m, Table~\ref{tab:ablation})
and the zero-shot absolute error on L2 ($3.21$\,m) are well below the
geoid--ellipsoid separation at those sites, which excludes a convention
mismatch of that magnitude.

\subsubsection{Baselines and Comparability}
Seventeen baselines span classical photogrammetry
~\cite{defranchis2014s2p,hirschmuller2008sgm,facciolo2015mgm}, supervised
satellite MVS~\cite{gao2021satmvs,gao2023satmvsf}, generalizable splatting
~\cite{charatan2024pixelsplat,chen2024mvsplat,xu2025depthsplat,tang2025hisplat,zhang2025transplat,huang2026skysplat},
zero-shot foundation models~\cite{wang2025vggt,lin2025da3,jiang2025anysplat},
and per-scene optimization
~\cite{mildenhall2020nerf,derksen2021snerf,mari2022satnerf,aira2025eogs,bournez2025eogspp}.
Concurrent methods without public code at submission
~\cite{chen2026satsurfgs,wagner2026rpcgs,satsplat2026,yang2026sat3r,luo2026eovggt}
are discussed but not re-implemented.

All methods receive the same fixed view triplets and are mapped to a common
geodetic height grid. Pinhole-native models use the best-fit local perspective
approximation to the RPC~\cite{zhang2019vissat}; non-height outputs are
converted by fixed family-level adapters. Rows that fit an affine height map to
the test reference are marked as oracle context and excluded from ranking.
When ground control is provided, every method receives the same seeded control
locations and the same correction basis of Proposition~\ref{prop:gauge}.
These choices preserve each method's native inference while making absolute
placement, completeness, and control budget comparable.

\subsubsection{Implementation}
Training uses $256^2$ crops, three views, and $K{=}96$ height hypotheses for
20k AdamW iterations with cosine decay (model selection at the
best dev-validation checkpoint, 16k). Rank-16 LoRA adapters ($\alpha{=}32$) are inserted into
the attention projections and upper-block MLPs of the frozen VGGT encoder;
MoGe-2 remains frozen. AdamW (weight decay $0.01$, gradient clipping at norm
$1$) uses learning rates of $2\times10^{-4}$ for the adapters and heads and
$1\times10^{-4}$ for the LoRA parameters, with a 1k-iteration linear warmup
before the cosine decay; batches hold two crops per GPU with two-step
gradient accumulation on 3--6 GPUs (effective batch $12$--$24$). The Ray
Transformer has $4$ layers, $8$ heads, and width $256$ (feed-forward $1024$)
over the $K{=}96$ hypotheses; the datum residual is a two-layer MLP of width
$64$ and the relief head a two-level UNet of base width $32$. The loss
weights $(\lambda_{\mathrm h},\lambda_{\mathrm{ray}},\lambda_{\mathrm{mono}},
\lambda_{\mathrm{anc}},\lambda_{\mathrm{nll}},\lambda_{\mathrm{grad}})$ are
$(1,0.5,0.2,0.2,0.1,0.05)$, with $\lambda_{\mathrm{nll}}$ ramped linearly over
the first 3k iterations.
Training batches are
$\delta$-shifted with probability $0.5$ ($\delta\sim\mathcal U(-10,10)$\,m),
which leaves the objective unchanged under the exact symmetry of
Eq.~\eqref{eq:shift} and exercises the consistency identities of
Sec.~\ref{sec:objectives} during training; height brackets are anchored to
the reference-view RPC using extents fixed from training data.
Model selection uses eight spatially disjoint development tiles from the
training region. Full tiles are decoded in the geodetic frame from $256$-pixel
windows. On one A800 GPU, \method{} requires $24.0$\,s of model-forward time
($89.5$\,s end to end) per L1 tile.

\begin{table*}[!t]
		\caption{Absolute DSM reconstruction on 26 held-out US3D tiles without ground control.}
		\label{tab:main}
		\centering\small
		\setlength{\tabcolsep}{4pt}
		\begin{tabular*}{\textwidth}{@{\extracolsep{\fill}}lccccccc@{}}
			\toprule
			& \multicolumn{3}{c}{Absolute placement} &
			\multicolumn{3}{c}{Accuracy--completeness} & Relief \\
			\cmidrule(lr){2-4}\cmidrule(lr){5-7}\cmidrule(lr){8-8}
			Method & Abs.\ MAE$\downarrow$ & Aln.\ MAE$\downarrow$ &
			Reg.$\uparrow$ & PAG$^{\mathrm{c}}_{2.5}\uparrow$ &
			PAG$^{\mathrm{c}}_{7.5}\uparrow$ & Cov$\uparrow$ & RF ($100$ ideal) \\
			& (m) & (m) & (\%) & (\%) & (\%) & (\%) & (\%) \\
			\midrule
			\multicolumn{8}{@{}l}{\emph{Feed-forward multi-view (primary
					track)}} \\
			DepthSplat~\cite{xu2025depthsplat} & 25.90 & 12.16 & \best{96.2} & 10.0 & 30.8 & 88.9 & 273.7 \\
			HiSplat~\cite{tang2025hisplat} & 33.68 & 7.90 & 53.8 & 15.8 & 43.5 & 81.4 & 233.0 \\
			MVSplat~\cite{chen2024mvsplat} & 101.96 & 3.49 & 3.8 & 25.9 & 61.9 & 84.0 & 38.1 \\
			pixelSplat~\cite{charatan2024pixelsplat} & 227.71 & -- & 0.0 & 21.5 & 51.5 & 79.7 & 131.6 \\
			SatMVS~\cite{gao2021satmvs} & 8.16 & 6.26 & 73.1 & \second{26.2} & \second{62.4} & \second{90.1} & 219.0 \\
			Sat-MVSF~\cite{gao2023satmvsf} & \second{4.70} & \second{3.36} & 80.8 & 7.8 & 11.1 & 12.4 & \best{108.5} \\
			SkySplat~\cite{huang2026skysplat} & 11.73 & 8.04 & \second{92.3} & 18.9 & 47.6 & 77.9 & 187.0 \\
			TranSplat~\cite{zhang2025transplat} & 49.71 & 25.08 & 11.5 & 4.2 & 10.9 & 68.1 & 610.7 \\
			\addlinespace[3pt]
			\multicolumn{8}{@{}l}{\emph{Context: classical CPU pipeline
					($\dagger$) and test-calibrated foundation models
					($\ddagger$)}} \\
			S2P$^{\dagger}$~\cite{defranchis2014s2p} & 15.59 & 4.89 & 50.0 & 19.3 & 37.1 & 47.4 & 36.0 \\
			AnySplat$^{\ddagger}$~\cite{jiang2025anysplat} & 12.39 & 6.59 & 69.2 & 25.9 & 62.5 & 93.0 & 99.6 \\
			DA3$^{\ddagger}$~\cite{lin2025da3} & 11.11 & 5.47 & 92.3 & 28.9 & 67.4 & 92.9 & 101.9 \\
			VGGT$^{\ddagger}$~\cite{wang2025vggt} & 11.83 & 6.11 & 92.3 & 27.2 & 62.1 & 93.1 & 96.5 \\
			\midrule
			\method{} (ours) & \best{2.99} & \best{1.78} & \best{96.2} & \best{72.6} & \best{89.7} & \best{91.9} & \second{88.4} \\
			\bottomrule
		\end{tabular*}
		
		\vspace{2pt}
\begin{minipage}{0.99\textwidth}
\footnotesize
Abs./Aln.: MAE before/after bounded 3-DoF translation, Aln.\ averaged over
tiles with an accepted alignment; Reg.: fraction of such tiles. PAG$^{\mathrm c}_{\tau}$ measures the fraction of
all reference pixels both predicted and within $\tau$ after diagnostic
alignment. Cov and RF are computed in the absolute frame. $^{\dagger}$Per-scene
context; $^{\ddagger}$test-reference calibration, excluded from ranking.
\end{minipage}
\end{table*}

\subsection{Main Results}
\label{sec:exp-main}

Table~\ref{tab:main} establishes the central result: preserving absolute
placement changes the ranking of satellite 3D systems. \method{} reaches
$2.99$\,m absolute MAE and $1.78$\,m aligned MAE at $91.9\%$ coverage, while
several perspective-native models keep reasonable local relief but carry large
vertical offsets. Its completeness-aware accuracies at $2.5$ and $7.5$\,m are
$72.6\%$ and $89.7\%$ (Fig.~\ref{fig:gallery}), $46.4$ and $27.3$ points above
the strongest compliant feed-forward baseline. Sat-MVSF reaches $4.70$\,m over
only $12.4\%$ of the surface, which is why error and completeness have to be
read together.

The gain is consistent across the test set. A tile bootstrap gives a $95\%$
interval of \GeoCI\,m for \method{}'s mean absolute error, and \method{} wins
all 26 tiles against \PairBase{} in both absolute error and
$\mathrm{PAG}^{\mathrm c}_{2.5}$. The paired absolute-error margin is
\CIabs\,m at $95\%$ confidence, and the advantage persists under sparse ground
control (Table~\ref{tab:gcp}). Fig.~\ref{fig:pareto} plots the same accuracy against model-forward time per
tile, obtained in one feed-forward pass with no scene-specific
optimization.

	\begin{figure}[!tb]
		\centering
		\includegraphics[width=\columnwidth]{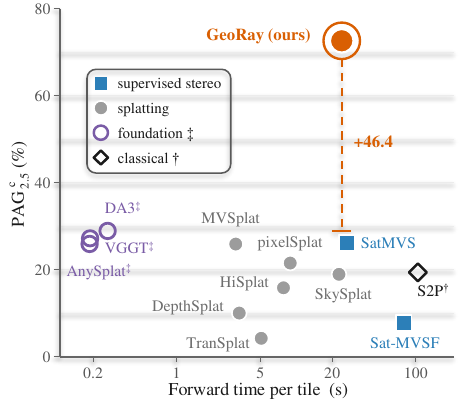}
		\caption{\textbf{Accuracy--efficiency trade-off on US3D.}
Completeness-aware accuracy is plotted against model-forward time per tile.}
		\label{fig:pareto}
	\end{figure}
	
	\begin{table}[!tb]
		\centering
		\caption{Sparse-control response on US3D under the shared affine correction model.}
		\label{tab:gcp}
		\footnotesize
		\setlength{\tabcolsep}{4pt}
		\begin{tabular*}{\columnwidth}{@{\extracolsep{\fill}}lcccc@{}}
			\toprule
			Method & $k{=}0$ & $k{=}1$ & $k{=}10$ &
			$\Delta_{\mathrm{gauge}}$ \\
			\midrule
			pixelSplat & 227.71 & 11.15 & 8.89 & $96.1\%$ \\
			MVSplat & 101.96 & 7.87 & 6.12 & $94.0\%$ \\
			HiSplat & 33.68 & 14.38 & 11.46 & $66.0\%$ \\
			DepthSplat & 25.90 & 18.06 & 13.40 & $48.3\%$ \\
			TranSplat & 49.71 & 37.96 & 32.13 & $35.4\%$ \\
			SkySplat$^{\ast}$ & 11.73 & 10.88 & 8.68 & $26.0\%$ \\
			\midrule
			\method{} (ours) & \best{2.99} & \best{2.14} & \best{1.82} &
			$39.2\%$ \\
			\bottomrule
		\end{tabular*}
		\begin{minipage}{\columnwidth}
\vspace{2pt}\footnotesize
All methods use identical control locations and the same correction basis.
$\Delta_{\mathrm{gauge}}$ is the fraction of zero-control error removed at
$k{=}10$; $^{\ast}$ denotes an RPC-native baseline.
\end{minipage}
	\end{table}
	
	\begin{figure*}[!tb]
		\centering
		\includegraphics[width=\textwidth]{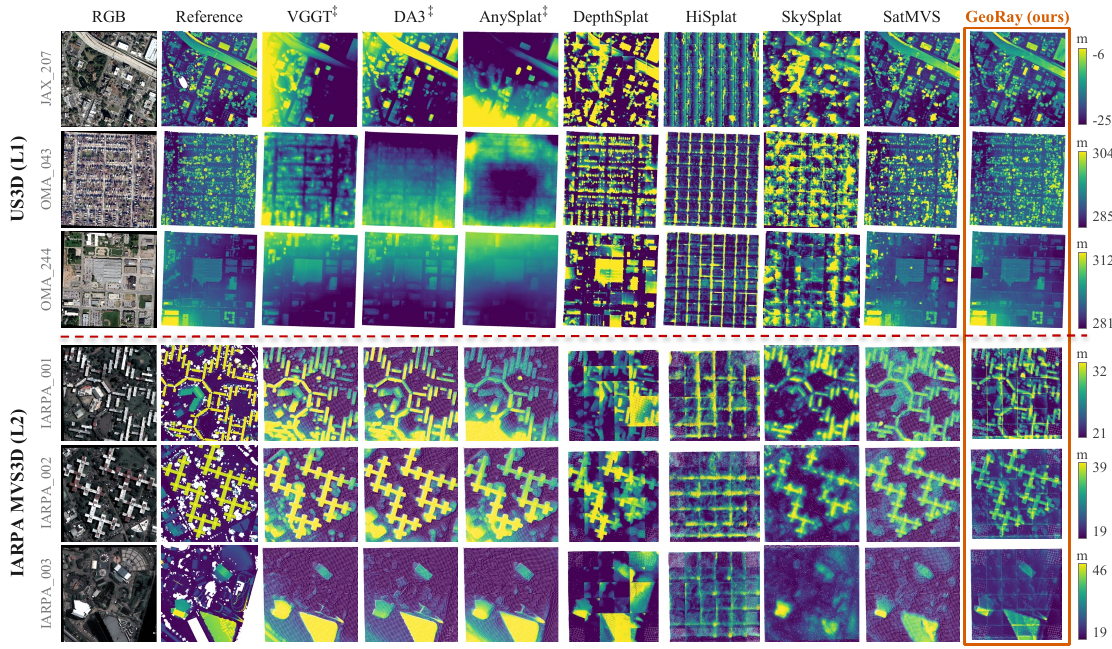}
		\caption{\textbf{Qualitative relief reconstruction on L1 and L2.}
All methods share the reference display range and per-tile median alignment so
the comparison emphasizes local relief and completeness.}
		\label{fig:gallery}
	\end{figure*}

	\begin{table*}[!t]
		\caption{Zero-shot transfer from US3D to IARPA MVS3D (L2) and ATL-SN4 (L3).}
		\label{tab:generalization}
		\centering\small
		\setlength{\tabcolsep}{3.5pt}
		\begin{tabular*}{\textwidth}{@{\extracolsep{\fill}}lcccccccccc@{}}
			\toprule
			& \multicolumn{5}{c}{IARPA MVS3D (L2)} &
			\multicolumn{5}{c}{ATL-SN4 (L3)} \\
			\cmidrule(lr){2-6}\cmidrule(lr){7-11}
			Method & Abs.$\downarrow$ & Aln.$\downarrow$ & Reg.$\uparrow$
			& Cov$^{\mathrm{g}}\uparrow$ & RF & Abs.$\downarrow$ & Aln.$\downarrow$ &
			Reg.$\uparrow$ & Cov$^{\mathrm{g}}\uparrow$ & RF \\
			& (m) & (m) & (\%) & (\%) & (\%) & (m) & (m) & (\%) & (\%) &
			(\%) \\
			\midrule
			\multicolumn{11}{@{}l}{\emph{Feed-forward multi-view (primary
					track)}} \\
			DepthSplat & 7.41 & 6.18 & \best{100.0} & \second{87.6} & 180.3 & 58.67 & 38.72 & \best{90.9} & 74.8 & 352.0 \\
			HiSplat & 17.72 & 6.91 & \second{33.3} & 82.6 & 123.4 & 99.73 & 28.59 & 13.6 & 55.9 & 296.6 \\
			MVSplat & 60.12 & 4.36 & \second{33.3} & 84.5 & 37.2 & 375.25 & -- & 0.0 & 80.1 & 52.1 \\
			pixelSplat & 6.58 & -- & 0.0 & 86.8 & 84.2 & \second{45.47} & \second{23.87} & 9.1 & 60.8 & 204.2 \\
			SatMVS & 5.66 & 5.19 & \best{100.0} & 79.2 & 103.1 & 115.11 & 49.10 & 31.8 & 71.7 & 360.4 \\
			Sat-MVSF & \second{3.37} & 3.33 & \best{100.0} & 39.1 & \second{100.8} & 50.01 & -- & 0.0 & 0.2 & 453.8 \\
			SkySplat & 5.07 & \second{3.25} & \best{100.0} & \best{96.0} & \best{100.6} & 69.48 & 37.11 & \second{63.6} & \second{83.5} & 283.2 \\
			TranSplat & 59.22 & 7.70 & \second{33.3} & 82.7 & 56.2 & 373.14 & -- & 0.0 & 77.3 & \second{59.1} \\
			\addlinespace[3pt]
			\multicolumn{11}{@{}l}{\emph{Context: per-scene ($\dagger$) and
					test-calibrated ($\ddagger$)}} \\
			S2P$^{\dagger}$ & 2.02 & 1.59 & 100.0 & 59.5 & 100.8 & 6.63 & 6.50 & 86.4 & 21.7 & 84.7 \\
			AnySplat$^{\ddagger}$ & 15.39 & 5.73 & 100.0 & 78.9 & 189.2 & 17.38 & 14.47 & 40.9 & 93.9 & 68.5 \\
			DA3$^{\ddagger}$ & 13.39 & 5.25 & 100.0 & 78.2 & 197.3 & 14.52 & 12.68 & 59.1 & 93.5 & 64.5 \\
			VGGT$^{\ddagger}$ & 13.79 & 5.59 & 100.0 & 78.3 & 199.6 & 17.06 & 13.83 & 50.0 & 93.9 & 67.8 \\
			\midrule
			\method{} (ours) & \best{3.21} & \best{2.86} & \best{100.0} & 86.7 & 81.0 & \best{23.88} & \best{15.74} & 45.5 & \best{90.7} & \best{118.7} \\
			\bottomrule
		\end{tabular*}
		\begin{minipage}{\textwidth}
\vspace{2pt}\footnotesize
Columns follow Table~\ref{tab:main}; Cov$^{\mathrm g}$ is grid coverage using
each tier's common scorer. Aligned MAE is reported only for tiles with an
accepted bounded translation; ``--'' indicates that none is accepted.
\end{minipage}
	\end{table*}

	\begin{figure*}[!tb]
		\centering
		\includegraphics[width=\textwidth]{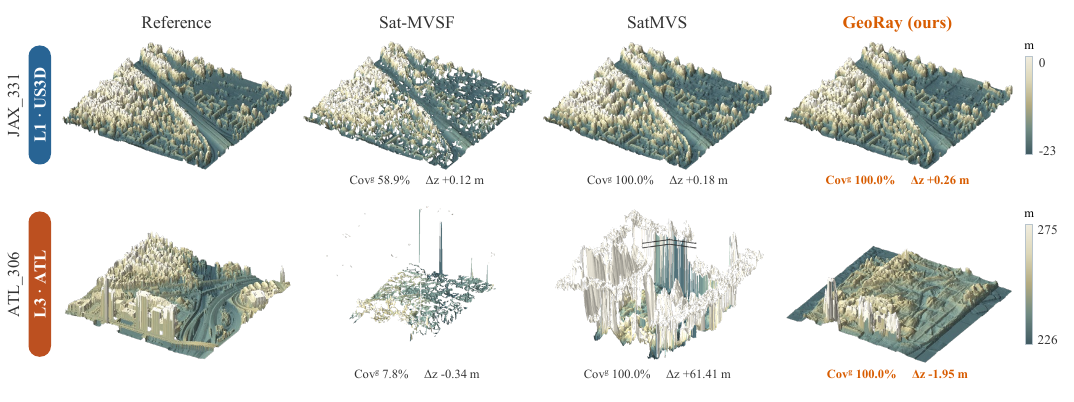}
		\caption{\textbf{Unregistered surfaces in the geodetic frame.}
Representative in-domain and cross-city reconstructions are rendered at their
predicted absolute elevations with a shared vertical scale within each row.}
		\label{fig:recon3d}
	\end{figure*}

\subsection{Generalization across Tiers}
	\label{sec:exp-generalization}
	
	Table~\ref{tab:generalization} evaluates unchanged models under two
out-of-distribution shifts. On L2, \method{} achieves the best absolute and
aligned error among compliant feed-forward methods ($3.21/2.86$\,m) at $86.7\%$
coverage. Sat-MVSF is close in absolute error ($3.37$\,m) but covers only
$39.1\%$, whereas SkySplat preserves more coverage but larger error. The L1 advantage thus survives a new region and an independent lidar
reference.

L3 jointly changes city, sensor, and acquisition geometry. Under this stronger
shift, \method{} retains $90.7\%$ coverage and the best compliant feed-forward
absolute/aligned error ($23.88/15.74$\,m), while supervised satellite MVS and
perspective-native systems either lose coverage, inflate relief, or retain
large datum error. Figs.~\ref{fig:recon3d} and~\ref{fig:gen} show the same pattern visually:
accuracy degrades under the shift, while the reconstructed surface stays
substantially more complete and better placed in the geodetic frame. We keep
the three tiers separate, since averaging distinct shifts into one number
would hide exactly this structure.

	\begin{figure}[!tb]
		\centering
		\includegraphics[width=\linewidth]{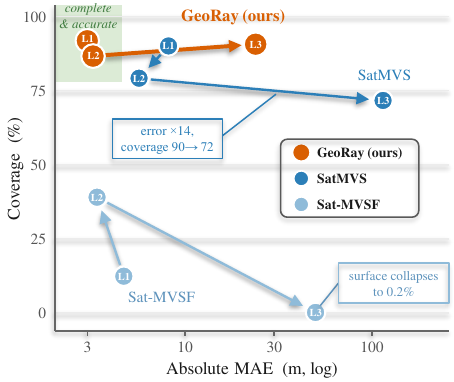}
		\caption{\textbf{Accuracy--coverage trajectories across the three tiers.}
Paths from L1 to L3 separate error growth at retained coverage from apparent
accuracy caused by coverage collapse.}
		\label{fig:gen}
	\end{figure}

	\subsection{Comparison with Per-Scene Reconstruction}
\label{sec:exp-perscene}

Per-scene optimization refits every scene independently and so provides an
accuracy--compute reference (Table~\ref{tab:perscene}). On the common
six-scene US3D subset,
\method{} reaches $2.49$\,m absolute error in $24.0$\,s of model-forward time
($89.5$\,s end to end), compared with $2.89$\,m in $1964$\,s for EOGS, while covering more of the
surface. On the three IARPA AOIs, EOGS is more accurate ($1.68$ versus
$3.21$\,m) but requires two orders of magnitude more computation. \method{} is thus a reusable feed-forward model that complements
scene-specific optimization, which stays preferable wherever maximum
per-scene accuracy justifies the extra compute.

\begin{table*}[!tb]
		\caption{Comparison with per-scene reconstruction on the common nine-scene subset.}
		\label{tab:perscene}
		\centering\small
		\setlength{\tabcolsep}{3.5pt}
		\begin{tabular*}{\textwidth}{@{\extracolsep{\fill}}lcccccccccc@{}}
			\toprule
			& \multicolumn{5}{c}{US3D subset (6 scenes)} &
			\multicolumn{5}{c}{IARPA (3 AOIs)} \\
			\cmidrule(lr){2-6}\cmidrule(lr){7-11}
			Method & Abs.$\downarrow$ & Aln.$\downarrow$ & Cov$\uparrow$
			& RF & Time$\downarrow$ & Abs.$\downarrow$ &
			Aln.$\downarrow$ & Cov$\uparrow$ & RF & Time$\downarrow$ \\
			& (m) & (m) & (\%) & (\%) & (s) & (m) & (m) & (\%) & (\%) &
			(s) \\
			\midrule
			S2P$^{\dagger}$~\cite{defranchis2014s2p} & 15.96 & 4.90 &
			45.5 & 30.9 & 123 & 2.02 & 1.59 & 59.5 & 100.8 & 32 \\
			NeRF$^{\dagger}$~\cite{mildenhall2020nerf} & 11.57 & 3.40 &
			99.4 & 7.8 & 16435 & 6.92 & 3.83 & 96.8 & 129.9 & 12035 \\
			S-NeRF$^{\dagger}$~\cite{derksen2021snerf} & 5.99 & 4.93$^{\ast}$ &
			99.3 & 57.6 & 28240 & 7.91 & 4.34 & 96.8 & 46.9 & 30721 \\
			Sat-NeRF$^{\dagger}$~\cite{mari2022satnerf} & 5.39 & 2.10$^{\ast}$ &
			99.3 & 72.8 & 30426 & 2.12 & 1.86 & 96.6 & 100.2 & 32598 \\
			EOGS$^{\dagger}$~\cite{aira2025eogs} & 2.89 & 2.89 & 32.7 &
			97.8 & 1964 & 1.68 & 1.47 & 100.0 & 99.9 & 1358 \\
			EOGS++$^{\dagger}$~\cite{bournez2025eogspp} & 3.01 & 2.44 &
			32.7 & 85.8 & 2781 & 1.98 & 1.39 & 100.0 & 99.0 & 1388 \\
			\midrule
			\method{} (ours) & 2.49 & 1.68 & 85.7 & 89.6 & 24.0\,/\,89.5 & 3.21 &
			2.86 & 86.7 & 81.0 & 9.1 \\
			\bottomrule
		\end{tabular*}
		
		\vspace{2pt}
		\begin{minipage}{0.99\textwidth}
\footnotesize
Abs. averages all scenes; Aln. averages scenes with accepted bounded
registration. Coverage follows each method's native output support; RF uses the
common valid intersection. Time is per-scene wall-clock on identical hardware; for \method{} the US3D
entry gives model-forward\,/\,end-to-end time and the IARPA entry
model-forward time. $^{\ast}$Aln.\ aggregated over the single US3D scene
passing the registration guard; the other five are excluded as unregistered.
\end{minipage}
	\end{table*}
\begin{figure*}[!t]
		\centering
		\includegraphics[width=\textwidth]{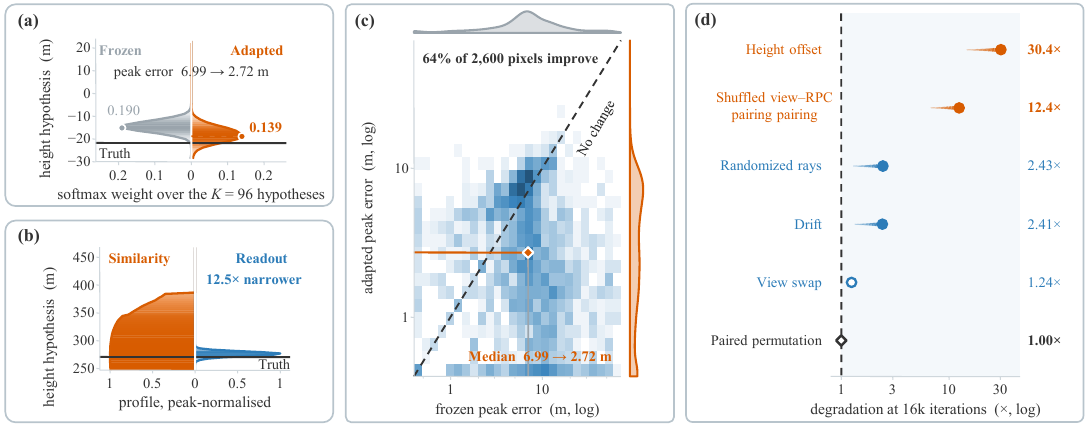}
		\caption{\textbf{Observability along RPC height rays.}
			Matching localization, posterior concentration, and camera
			counterfactuals isolate the effect of ray-consistent adaptation.}
		\label{fig:observability}
	\end{figure*}

\subsection{Mechanism I: RPC Ray-Field Observability (C1)}
	\label{sec:exp-observability}
	
	Frozen perspective-pretrained features are not reliably observable
	along RPC height rays, and lightweight ray-consistent adaptation
	restores the property. This section characterizes the intervention at the feature
		level, whereas the ablation table reports it as an end-to-end metric
		delta, and the two readings answer different questions. All mechanism analyses
		here and in the two sections that follow use the production model
		(Table~\ref{tab:evalmatrix}); the reduced-budget protocol applies
		only to the ablations.
	
	\emph{Protocol.} On a fixed, deterministically sampled set of
	$2{,}600$ pixels spanning building edges, rooftops, and ground across
	all test tiles, we trace each pixel's RPC height ray and record the
	cross-view feature-cosine profile over the $K$ hypotheses, under an
	identical similarity metric and normalization for frozen and
	adapted features; the ray-readout posterior is recorded separately,
	so feature-level and head-level effects are never conflated.
	
	\emph{Findings} (Fig.~\ref{fig:observability}(a)--(c)). Frozen profiles
	peak a median $6.99$\,m from the true height at contrast $0.192$;
	adaptation moves the peak to a median $2.72$\,m (P90 $8.62$\,m), a
		$2.6{\times}$ localization gain, while the raw contrast falls to
		$0.138$. Higher contrast in the frozen features is thus no evidence of
		better localization: it accompanies larger per-pixel DSM error
		(Spearman $r_{\mathrm{s}} = 0.36$). The ray readout supplies the
		final posterior sharpness, so adaptation corrects localization and
		the readout sharpens the posterior.
	
	\emph{Counterfactuals} (Fig.~\ref{fig:observability}(d)). Five RPC corruption
	families (shuffled view--RPC pairing, height offset, drift,
	randomized rays, and view swap) each break the correspondence
	between images and cameras; reconstruction degrades by
	$1.2$--$30.4{\times}$, and for four of the five the degradation
	grows monotonically over training. A sixth family permutes the views
	\emph{together with} their RPCs, so every image--camera pair survives
	and the geometry is untouched: it costs nothing
		($1.00{\times}$). This paired-view permutation is the control the analysis
			requires: what degrades reconstruction is corrupted
			camera geometry, not perturbed input, so the restored competence
			is geometric and not an effect of memorized appearance.

	\begin{figure*}[!t]
		\centering
		\includegraphics[width=\textwidth]{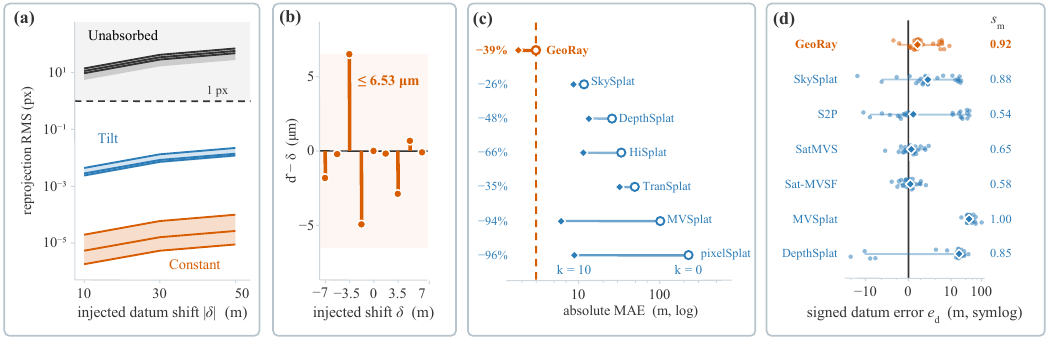}
		\caption{\textbf{The height--datum gauge, measured.}
			Injected-shift response, per-camera gauge spectra, and
			sparse-control behavior establish the low-order structure of
			absolute error.}
		\label{fig:gaugeexp}
	\end{figure*}

\subsection{Mechanism II: The Height--Datum Gauge (C2)}
	\label{sec:exp-gauge}
	
	The gauge behaves as constructed: the datum estimate tracks an
		injected shift with slope $1.00$, a single control point removes most
		of the residual scalar datum error, and what remains is a low-order
		field. We establish each claim in turn.
	
	\emph{Validity of the first-order gauge.} Proposition~\ref{prop:gauge}
	is an approximation, and its quality is a property of the cameras, so
	we measure it on them with no model in the loop. For each tile and
	each view we inject the modes of $\mathcal{B}$ at amplitudes
	$\delta \in \{\pm 10, \pm 30, \pm 50\}$\,m, form the exact
	reprojection displacement field, and fit the admissible bias update of
	Eq.~\eqref{eq:obsmodel} to it in least squares; what is left is $R$ of
	Eq.~\eqref{eq:gaugeresid}. At $\delta = 50$\,m the constant mode
	leaves $R = \GaugeResidPx$\,px against a displacement field of
	$\GaugeDispPx$\,px, and the affine modes leave at most
	$\GaugeResidAff$\,px at the worst tile. Divided by the disparity rate, the
	\emph{constant-mode} residual is $\GaugeResidM$\,m of observable
	height out of the $50$\,m injected: multi-view evidence, read at any
	precision these cameras support, constrains a part in $10^{6}$ of a
	scalar datum shift, and the tilt modes, whose residual is two orders
	larger, remain far below any usable matching precision. That the
	number is small is only meaningful against a control, so we repeat
	the fit for a perturbation \emph{outside} $\mathcal{B}$, which the
	bias model has no witness for: it leaves a residual five to six
	orders of magnitude larger on every tier
	($1.2{\times}10^{5}$ to $1.7{\times}10^{6}$ times the constant-mode
	value). The separation is wide enough to identify the affine family as a genuine
near-null direction of these cameras and not an artifact of the fit. The measured residual sits far below the $O(\eta)$ bound of
Proposition~\ref{prop:gauge}(i) because the
	bias design absorbs the affine part of the variation of
	$\mathbf{j}_i$ as well, leaving second-order terms; the bound is an
	upper bound and these cameras are better behaved than it requires.
	The same measurement on the strongly oblique cross-city cameras
	($R = \GaugeResidLthree$\,px) is what
	Sec.~\ref{sec:exp-failure} reads its residual datum field against.
	
	\emph{Gauge tracking.} We vary the RPC height origin over the tested
$\delta\in[-7,7]$\,m sweep while leaving the images fixed. The datum
estimate follows with unit slope to numerical precision, with maximum
residual below $6.53\,\mu$m (Fig.~\ref{fig:gaugeexp}(b)). The same equivariance
holds for the full prediction, not only the datum head: over the same sweep the
fused height field follows the injection with slope $0.9999$, its
displacement from the shifted reference staying within $0.008$\,m at the
$95$th percentile. The monocular relief is unchanged to the reported
precision, as required by the coordinate-origin separation in
Sec.~\ref{sec:gauge}. The experiment is a direct held-out verification of the constructive
equivariance and reports the held-out value of the consistency identities of
Sec.~\ref{sec:objectives}; it makes no extrapolation claim beyond the
$\pm10$\,m training-shift interval.

\emph{Control response, and what it reveals about the field.}
	Table~\ref{tab:gcp} applies the correction family of
		Proposition~\ref{prop:gauge} identically to every method. The last column
		estimates how much of each system's absolute error is attributable to the
		low-order datum component rather than the remaining relief error.
	The two systems with the largest absolute errors, both pinhole-native, place
$94$ and $96\%$ of that error in the datum ($101.96 \to 6.12$ and
$227.71 \to 8.89$\,m). Most of their error is thus a matter of absolute placement, not local
relief, exactly as expected when the low-order gauge dominates the
unregistered prediction. The
	RPC-native baseline leaves the smallest share ($26.0\%$), consistent with
		its explicit treatment of RPC geometry. TranSplat shows the opposite
		pattern: only about one third of its error is removable, indicating that
		relief error dominates. Since the reported share is a ratio, it describes error
		composition and not error magnitude. The gauge belongs to the
		observation model and is inherited by every method: absolute
		satellite reconstruction is limited by identifiability as much as by
		model capacity. Even after correction, the best-responding baseline, \BaselineKtenName{},
ends at $\BaselineKten$\,m versus \method{}'s zero-control $2.99$\,m. The
ordering is also unchanged on both shift tiers, where HiSplat reaches
$\BaselineShift$\,m under the same ten points.
	
	\method{}'s own response separates the two control regimes as the
	analysis predicts. One point fixes the scalar gauge and recovers
	$70.7\%$ of the remaining absolute--aligned gap
	($2.992 \to 2.137$\,m against the $1.782$\,m aligned floor); ten constrain the affine field
	($1.82$\,m). Under transfer the regimes diverge: on L2, where the
	zero-control datum gap is already $0.35$\,m and below the noise of a
	single control height, one point cannot help ($3.21 \to 3.99$\,m)
	while ten distributed points do ($3.13$\,m), whereas on L3 the scalar correction provides limited improvement for the
smooth low-order field of Sec.~\ref{sec:exp-failure}
($23.88 \to 22.09$\,m), while the affine correction reduces the error to
$16.94$\,m, recovering $74\%$ of the affine-oracle gain.

\emph{Population view.} Per-tile signed datum errors separate the
	families (Fig.~\ref{fig:gaugeexp}(d)): pinhole-derived methods
	scatter over tens of meters, supervised RPC-MVS depends on its
	height prior, and \method{} concentrates near zero.

	\begin{figure*}[!t]
		\centering
		\includegraphics[width=\textwidth]{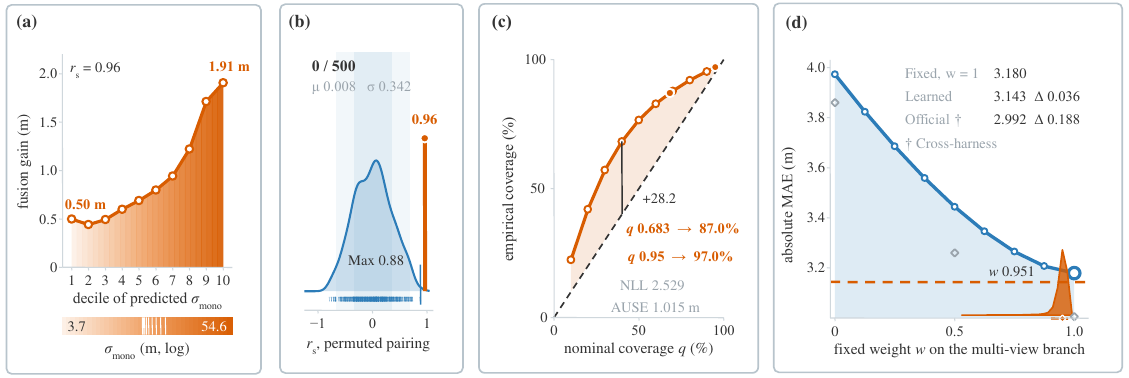}
		\caption{\textbf{Calibration and the conditional gain law.}
			Reliability, interval coverage, permutation and fixed-weight
			controls test where multi-view evidence should dominate.}
		\label{fig:gainlaw}
	\end{figure*}

\subsection{Mechanism III: Calibration and the Conditional Gain
		Law (C3)}
	\label{sec:exp-gain}
	
	The fusion of Eq.~\eqref{eq:fusion} requires uncertainties that are
		useful both for ranking and for scale. On held-out data, predicted
uncertainty yields $\mathrm{NLL}=2.53$ and sparsification error
$\mathrm{AUSE}=1.02$\,m~\cite{ilg2018uncertainty}
(Fig.~\ref{fig:gainlaw}(c)), with conservative empirical interval coverage
and spatial agreement with realized error. These results license the use of predicted uncertainty as the conditioning
variable below.
	
	\emph{Calibration.} Sparsification evaluates the \emph{ordering} of
		predicted uncertainty; calibration of its numerical scale is a separate
		question, and that scale enters the fusion weights of
		Eq.~\eqref{eq:fusion}. We report interval coverage against
		nominal~\cite{gneiting2007calibration,kuleshov2018calibrated},
	read on the median-centered residual that $\widehat{\sigma}$ models,
	with the absolute-frame result reported alongside it so that per-tile datum
error remains visible. The intervals are conservative, not exact. The nominal $68.3\%$ and $95.0\%$ levels cover
	$\CovOne\%$ and $\CovTwo\%$ at a variance-scaling factor
	$\mathrm{RMS}(e/\widehat{\sigma}) = \VarScale$, indicating conservative
		scale calibration rather than overconfidence. The excess coverage is
		concentrated at the narrower interval and nearly vanishes at the wider one.
		The pattern points to non-Gaussian residuals, more sharply
		peaked near zero but heavier tailed, so a single variance rescaling
		cannot match both coverage levels at once. The calibration head acts on the fusion weights and on the ordering
	(its removal costs $0.031$--$0.198$\,m across the tiers, below) rather
	than on the interval scale, which the heteroscedastic likelihood already
	sets ($\VarScaleU$, at $\CovOneU\%$ and $\CovTwoU\%$, without the head). In the absolute frame, the same intervals cover $\CovOneAbs\%$ and
$\CovTwoAbs\%$. The lower coverage reflects the datum component that
$\widehat{\sigma}$ does not model and that Fig.~\ref{fig:gaugeexp}(d)
reports across the test population. The measured
		over-coverage is conservative despite the shared-image dependence between the two streams, so the independence approximation
does not create an overconfident interval in practice. A datum-consistent
sweep over fixed global weights then provides a direct reference for the
learned gate (Fig.~\ref{fig:gainlaw}(d)). Its best fixed
choice lies at the boundary, $w=1$, with $3.180$\,m MAE; under the same
protocol, the learned gate reaches $3.143$\,m, a $0.036$\,m gain.
The sweep scores every $w$ on the joint valid-pixel mask of the three
streams with the per-tile datum held fixed, whereas the main evaluation protocol scores the fused raster on its own
support; the same predictions therefore carry a constant $\approx0.19$\,m
offset at every $w$. The reported $2.992$\,m result follows the main-table protocol and is kept
separate from the fixed-weight sweep. The learned weight distribution is concentrated near
	the same boundary (median $0.951$, mode $0.95$), so the learned weights remain
	close to the fixed-weight optimum while retaining the pixelwise variation
	tested by the conditional law below.
	
	\emph{The conditional law.} Stratifying the fusion gain of
	Eq.~\eqref{eq:gain} by \emph{predicted} monocular uncertainty
	yields a profile that rises with predicted uncertainty from
	$0.50$\,m in the most confident decile to $1.91$\,m in the least, a
	$3.8{\times}$ rise with rank correlation $0.96$ across the ten
	deciles (Fig.~\ref{fig:gainlaw}(a)). A permutation null tests whether this ordering could arise from the
	marginal distributions alone. Each of $\NullDraws$ repetitions
	randomly permutes the pixel pairing between predicted monocular uncertainty
and realized fusion gain over the full leakage-controlled pixel set,
	then recomputes the decile profile and its rank correlation. The null
	centers at $\NullMean \pm \NullSd$, no draw reaches the observed
	$0.96$ ($0/\NullDraws$ exceed; plus-one Monte-Carlo $p=1/(\NullDraws+1)\approx0.002$;
	Fig.~\ref{fig:gainlaw}(b)). The uncertainty--gain ordering is therefore a population-level
		association and not an artifact of the decile marginals.
	
	\emph{Disparity rate acts at acquisition scale.} Binning the same pixels by
$\rho_{\mathrm{eff}} = (\sum_i \rho_i^{2})^{1/2}$, the precision with
which a triplet resolves height, initially suggests the expected geometric
trend, but that trend disappears after controlling for tile identity. An analysis of variance places essentially
	all of the variation of $\rho_{\mathrm{eff}}$ between tiles and none
	within them, so binning pixels by it bins them by tile identity;
	removing per-tile means leaves the marginal at $\GainRhoWithin$\,m
		and non-monotonic. We therefore state the law on predicted uncertainty alone. The
			controlled result is still informative: the near-constancy of
			$\rho_{\mathrm{eff}}$ within a tile supports the magnitude
			component of the tile-scale regime assumed in
			Proposition~\ref{prop:gauge}. The law's structure transfers across domains: under the L2
	shift the multi-view gain stays concentrated in the least-confident
	pixels (Spearman $0.98$), with a top-decile gain of $1.0$\,m
	against $1.9$\,m in domain\,---\,the reduced magnitude reflecting
	that the multi-view geometry itself degrades under the same shift.
	L2 also has the weakest view convergence of the three tiers
($17.5^{\circ}$ median versus $31.2^{\circ}$ and $41.0^{\circ}$,
Fig.~\ref{fig:benchmark}(b)), providing a direct geometric explanation for
the reduced gain. A further consequence is that the contribution of the
			fusion components changes as geometry weakens.
		Switching the monocular stream off costs $0.014$, $0.056$, and $0.554$\,m
		in domain, on L2, and on L3; bypassing calibration costs $0.031$, $0.198$,
		and $0.134$\,m. Relative to each tier's error scale, the monocular stream
		increases steadily in importance ($0.5$, $1.7$, $2.3\%$), whereas the
		calibration effect is largest on L2 ($1.0$, $6.2$, $0.6\%$ across L1--L3).
		This pattern is consistent with L2 having the weakest convergence while
		retaining a comparatively accurate surface; on L3, large residual errors
		dominate the uncertainty-scale effect. Both components therefore matter most where multi-view evidence is weak, as
the conditional law predicts. What the experiments establish is the
conditional relationship, not a uniform mean shift. Since the law is stated on
quantities available before reconstruction, it can also signal when additional
multi-view acquisition is likely to pay off.

	\begin{table}[!tb]
		\caption{Ablations on the 26 held-out US3D tiles.}
		\label{tab:ablation}
		\centering\footnotesize
		\setlength{\tabcolsep}{3pt}
		\begin{tabular}{@{}lcccc@{}}
			\toprule
			Configuration & Abs.$\downarrow$ &
			PAG$^{\mathrm{c}}_{2.5}\uparrow$ & Cov$\uparrow$ & RF \\
			& (m) & (\%) & (\%) & (\%) \\
			\midrule
			\multicolumn{5}{@{}l}{\emph{Inference ablations on the final
					model}} \\
			\method{} (final) & 2.99 & 72.6 & 91.9 & 88.4 \\
			datum head, anchor only ($r_{\psi} \equiv 0$) & \AnchorAbs &
			\AnchorPag & \AnchorCov & \AnchorRf \\
			w/o datum head ($\widehat{\Delta d} \equiv 0$) & 150.84 & 47.1 &
			89.8 & 71.9 \\
			uncalibrated fusion & 3.02 & 71.9 & 92.0 & 86.2 \\
			two views & 3.34 & 68.9 & 92.0 & 84.8 \\
			\addlinespace[2pt]
			\multicolumn{5}{@{}l}{\emph{Fusion-arm checks (inference-time switches)}} \\
			multi-view only & 3.01 & 72.5 & 91.9 & 89.9 \\
			monocular only & 3.86 & 57.5 & 92.5 & 71.8 \\
			fixed-weight model ($w \equiv 0.5$) & 3.26 & 67.6 & 92.3 & 78.6 \\
			\addlinespace[2pt]
			\multicolumn{5}{@{}l}{\emph{Matched-compute retraining}} \\
			Full model (matched compute) & 3.59 & 69.1 & 91.7 & 86.2 \\
			w/o ray-field adaptation & 20.47 & 26.4 & 90.7 & 115.4 \\
			multi-view only (retrained) & 7.45 & 42.7 & 92.1 & 70.1 \\
			\bottomrule
		\end{tabular}
	\end{table}

\begin{table}[!tb]
		\caption{Monocular-prior swap under matched compute.}
		\label{tab:backbones}
		\centering\footnotesize
		\setlength{\tabcolsep}{4pt}
		\begin{tabular}{@{}llcc@{}}
			\toprule
			Monocular prior & Budget & Abs.\ MAE$\downarrow$ &
			PAG$^{\mathrm{c}}_{2.5}\uparrow$ \\
			\midrule
			MoGe-2~\cite{wang2025moge2} & production & 2.99 & 72.6 \\
			MoGe-2 & matched & 3.59 & 69.1 \\
			Depth-Anything-V2~\cite{yang2024depthanything2} & matched & 3.58 & 69.0 \\
			\bottomrule
		\end{tabular}
	\end{table}

\subsection{Ablations and Sensitivity}
\label{sec:exp-ablation}

Table~\ref{tab:ablation} separates large structural effects from local design
choices. Removing ray-field adaptation in matched retraining raises absolute
MAE from $3.59$ to $20.47$\,m, confirming that perspective-pretrained features
must be adapted to native RPC rays. Removing the datum mechanism from the final
model raises error from $2.99$ to $150.84$\,m; retaining only the analytic
anchor gives $\AnchorAbs$\,m, showing that the analytic construction provides
most of the absolute placement while the learned residual refines it. In the matched retraining arm, removing the monocular stream raises error
from $3.59$ to $7.45$\,m. The two arms address different regimes. With a converged multi-view branch,
in-domain multi-view evidence is already strong and the average effect of the
monocular prior is correspondingly small ($0.014$\,m); under matched compute
it compensates for an undertrained multi-view branch ($3.59$ against $7.45$\,m
without it). Its value concentrates where multi-view evidence is weak, whether
from a limited training budget or, as on L3 (Sec.~\ref{sec:exp-gain}), from
weak geometry.
The three fusion-arm rows are inference-time switches on the shared
final checkpoint ($w{\equiv}1$, $w{\equiv}0$, $w{\equiv}0.5$); no
retraining is involved.

The remaining changes are comparatively small: uncalibrated fusion gives
$3.02$\,m, two views give $3.34$\,m, and varying the hypothesis count over
$K\in[48,144]$ moves absolute error by less than $0.1$\,m.
Table~\ref{tab:backbones} swaps the monocular prior under the same matched
budget, where Depth-Anything-V2 and MoGe-2 give $3.58$ and $3.59$\,m, so the
result does not rest on one particular prior. The inference-time single-stream
and fixed-weight switches remain competitive in domain, as expected of C3,
which is a conditional reliability claim and not a claim of a large uniform
average gain. We do not interpret differences among closely matched variants;
the ablation is designed to expose large effects, not to rank configurations
separated by marginal numbers.

	\begin{figure}[!tb]
		\centering
		\includegraphics[width=\linewidth]{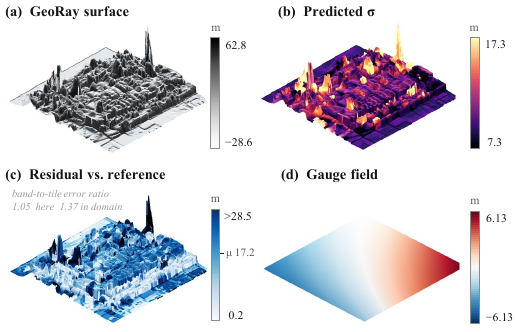}
		\caption{\textbf{Residual structure under cross-city transfer.}
Predicted uncertainty and the fitted low-order gauge separate local relief error
from the remaining datum component on ATL\_554.}
\label{fig:failure}
	\end{figure}

\subsection{Scope and Outlook}
\label{sec:exp-failure}

Fig.~\ref{fig:failure} characterizes the residual structure under the strongest cross-city shift. Under the
strong L3 shift, the residual changes from predominantly local structure error
to a smoother low-order field; sparse affine control reduces the tier error
from $23.88$ to $16.94$\,m and recovers $74\%$ of the corresponding affine
oracle gain. Predicted uncertainty highlights difficult regions, while the remaining
low-order datum component is better addressed by geometric control; the two
mechanisms are complementary.

The present formulation has two main limitations. The gauge basis is
intentionally low order, and the uncertainty model describes relief but not
the per-tile datum. Richer sensor-error bases and joint relief--datum
uncertainty are natural extensions. Variable-cardinality ray aggregation and
coarse-to-fine height sampling are engineering extensions for large-area
deployment. Overlap-and-blend window decoding can remove the faint window
boundaries visible in some L2 tiles of Fig.~\ref{fig:gallery} without changing
the observability or gauge formulation.

\section{Conclusion}
\label{sec:conclusion}

Satellite 3D reconstruction in the absolute geodetic frame is not
perspective reconstruction at a different scale. Native RPC imaging creates
three coupled requirements: correspondence must be observable along
object-space height rays, absolute elevation must be separated from a
low-order height--datum gauge, and monocular and multi-view evidence must be
combined according to calibrated reliability. \method{} addresses these
requirements explicitly through ray-consistent adaptation, a
coordinate-equivariant datum mechanism, and precision-weighted relief fusion;
\bench{} provides an evaluation in which absolute placement and completeness
remain visible.

Across in-domain, cross-dataset, and cross-city tests, one trained model
produces dense geodetic surfaces without per-scene optimization and leads the
compliant feed-forward methods in absolute accuracy. The mechanism experiments add matching evidence: adaptation restores
localization along native RPC rays, the measured camera gauge explains why
sparse control acts mainly on absolute level, and uncertainty predicts where
multi-view geometry contributes most. They point to a design principle for
geometric foundation models under non-central sensing: preserve the native
camera model, expose the gauge of the observation model, and learn only the
degrees of freedom the measurements support. We expect the same formulation to
carry over to other non-central or rolling-shutter sensors, to multi-date and
multi-sensor constellations, and to radar--optical fusion, wherever a reusable
geometric prior has to coexist with low-dimensional, physically interpretable
sensor uncertainty.

\section*{Acknowledgment}
	This work was supported by the National Natural Science Foundation
	of China under Grant 624B2051. The authors thank the providers of
	the public datasets underlying \bench{}: the Johns Hopkins
	University Applied Physics Laboratory and the IEEE GRSS Image
	Analysis and Data Fusion Technical Committee (US3D/DFC2019), IARPA
	(MVS3D), the SpaceNet partners (SN4), and Maxar/DigitalGlobe for
	the underlying imagery. 
	
	\bibliography{refs}
	
\end{document}